\documentclass{article}

 \usepackage[preprint]{neurips_2026}

\usepackage[utf8]{inputenc} 
\usepackage[T1]{fontenc}    
\usepackage{url}            
\usepackage{booktabs}       
\usepackage{amsfonts}       
\usepackage{nicefrac}       
\usepackage{microtype}      
\usepackage[dvipsnames]{xcolor}         

\usepackage{amsthm}

\theoremstyle{remark}
\newtheorem{remark}{Remark}
\theoremstyle{definition}
\newtheorem{definition}{Definition}

\usepackage{graphicx}
\usepackage{float}
\usepackage{amsmath}
\usepackage{algorithm}
\usepackage{algorithmic}
\usepackage{amssymb}
\usepackage{multirow}
\usepackage{colortbl}    
\usepackage{pifont}      
\usepackage{adjustbox}   
\usepackage{hyperref}
\usepackage[most]{tcolorbox}
\newcommand{\cmark}{\textcolor[HTML]{2E8B57}{\ding{51}}}  
\newcommand{\xmark}{\textcolor[HTML]{888888}{\ding{55}}}  
\newtheorem{theorem}{Theorem}
\newtheorem{corollary}[theorem]{Corollary}
\newtheorem{lemma}[theorem]{Lemma}
\newtheorem{proposition}[theorem]{Proposition}
\newtheorem{assumption}[theorem]{Assumption}

\definecolor{LightGreenBox}{HTML}{EAF5EA}

\usepackage[table]{xcolor}
\definecolor{LightBlueBox}{HTML}{EAF3FF}

\title{R2V Agent: Teaching SLMs When to Ask for Help}

\author{%
\textbf{Raghu Vamshi Hemadri}$^1$\thanks{Both authors contributed equally to this research.} \quad 
\textbf{Humaira Firdowse Mohammed}$^2$\footnotemark[1] \quad 
\textbf{Rishabh Maheshwary}$^4$ \\  
\textbf{Srivatsava Daruru}$^3$ \quad
\textbf{Sagar Davasam}$^4$ \quad
\textbf{Vikas Yadav}$^4$ \quad
\textbf{Srinivas Sunkara}$^4$ \quad 
\textbf{Sai Rajeswar}$^{4,5,6}$ \\ \\
$^1$New York University Tandon School of Engineering \quad
$^2$University of California, San Diego\\
$^3$Stanford University \quad
$^4$ServiceNow Research \quad
$^5$Mila - Quebec AI Institute \quad
$ ^6$Université de Montréal
}

\begin{document}

\maketitle

\begin{abstract}
Efficient agentic systems should incur expensive frontier-model costs only on decisions where a cheaper local model is likely to fail. Existing LLM cascades usually route whole queries before execution, but task difficulty shifts mid-trajectory - after flaky tool calls, truncated observations, or compounding local errors - making pre-execution routing brittle. We introduce \textbf{R2V-Agent}, a risk-calibrated SLM-LLM routing framework for interactive agents. R2V combines four components: a distilled small language model (SLM) policy, a stronger teacher LLM, a lightweight process verifier that scores candidate actions at each step, and a calibrated step-level router. The router is our central contribution: after the SLM is trained, it estimates residual failure risk at each step and escalates only when teacher intervention is warranted. To make the routing problem well-defined, we first train a stable local SLM using a standard offline pipeline: behavioral cloning (BC) on teacher trajectories, followed by verifier-guided Direct Preference Optimization (DPO) with consistency regularization. The router is then trained on this fixed policy's residual failures using Brier-calibrated probability estimation and a Conditional Value-at-Risk (CVaR)-constrained objective that penalizes worst-case failures across perturbation seeds. Across HumanEval+, TextWorld, and TerminalBench with four SLM backbones, R2V improves the reliability-cost frontier: it achieves $94.3\%$ HumanEval+ success with $0.60\%$ LLM escalation, recovers TextWorld from $64.6\%$ SLM-only success to $98.2\%$ at $41.7\%$ escalation, and reaches $93.3\%$ TerminalBench success at $33.9\%$ LLM calls, roughly half the heuristic-router cost.
\end{abstract}

\section{Introduction}

Autonomous language-model agents increasingly operate in interactive settings such as software engineering \citep{devin_2024, swe_agent_2024}, web navigation \citep{webarena_2024}, tool use, and system administration. Reliability in these environments often depends on expensive frontier LLMs, especially when tasks require long-horizon planning, recovery from failed tool calls, or reasoning under incomplete observations \citep{gaia_2024}. At the same time, inference-time scaling and process supervision suggest that additional computation is most valuable when allocated to the right intermediate decisions rather than spent uniformly across all steps \citep{snell_testtime_2024, xi_agentprm_2025}. This raises a natural systems question: can a cheap local SLM execute routine agent steps while a stronger LLM is invoked only when the current decision is genuinely high risk?

\begin{figure}[t]
    \centering
    \includegraphics[width=0.77\textwidth]{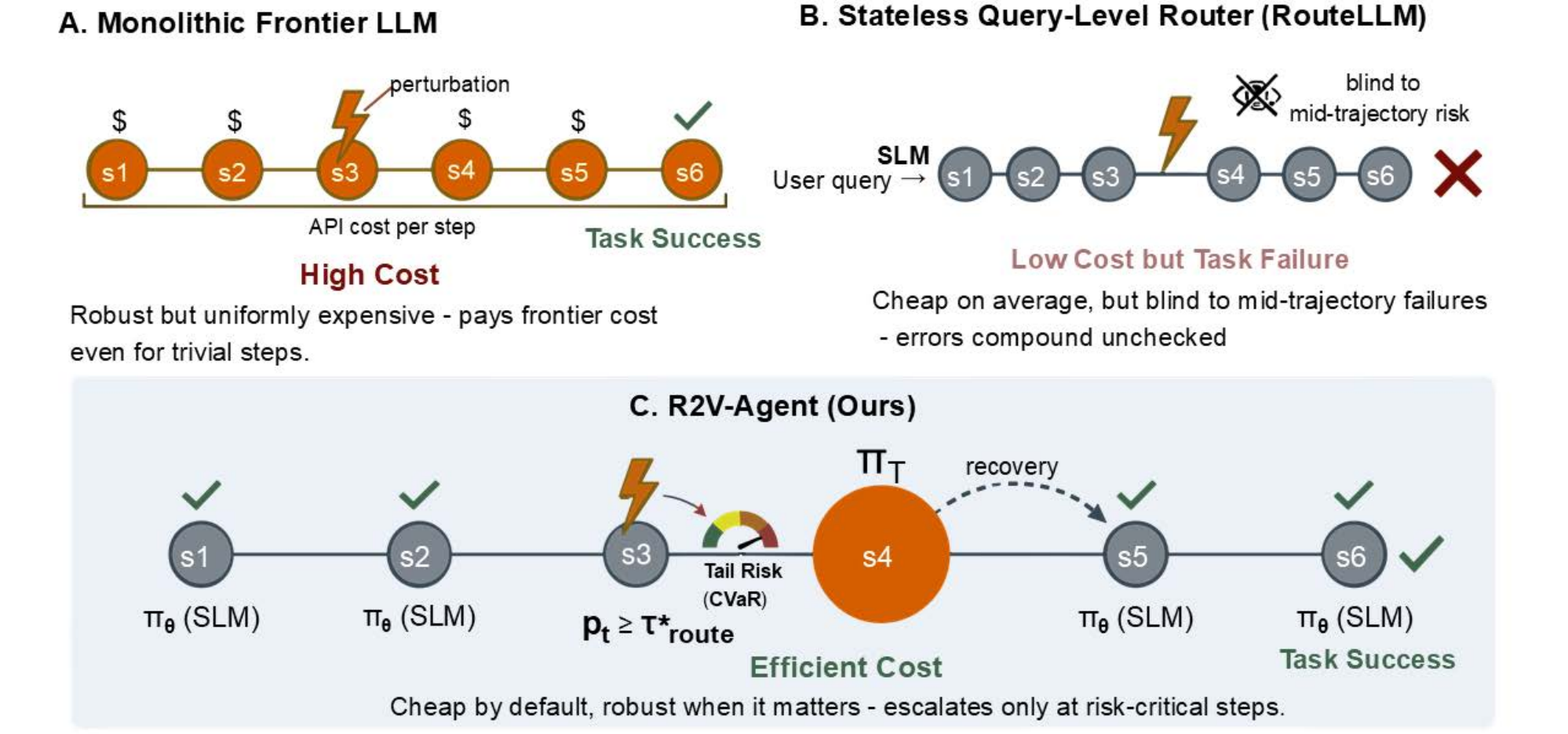}
    \caption{Traditional agentic workflows versus R2V-Agent. Monolithic frontier-LLM execution is robust but expensive; stateless query-level routing is cheaper but blind to mid-trajectory risk; R2V performs step-level escalation only when verifier-grounded risk exceeds the routing threshold.}
    \label{fig:r2v_overview}
\vspace{-0.5em}
\end{figure}

Existing LLM cascades and routers make progress toward cost reduction, but most are designed for static, query-level allocation. Systems such as FrugalGPT and RouteLLM decide which model should answer a request before execution, or cascade after a full response has already been generated \citep{frugalgpt_2023, routellm_2024, dekoninck_unified_2024, wang_hybrid_2024}. This abstraction is poorly matched to agentic POMDPs, where difficulty is trajectory-dependent. A task may begin in-distribution, then become risky after a flaky tool call, truncated observation, prompt injection, or compounding local error. Routing the whole task to the SLM is brittle, while routing the whole task to the LLM is unnecessarily expensive. Figure~\ref{fig:r2v_overview} illustrates the difference: monolithic frontier execution is robust but costly, query-level routing is cheap but blind to mid-trajectory failures, and R2V-Agent escalates only at risk-critical steps.

We propose \textbf{R2V-Agent}, a risk-calibrated framework for step-level SLM-LLM routing in interactive agents. R2V factorizes execution into four components: an efficient SLM policy $\pi_\theta$, a stronger teacher LLM $\pi_T$, a lightweight process verifier $V_\phi(x_t,a_t)$ that scores candidate actions, and a router $r_\psi(f_t)$ that estimates whether the current step should be escalated. The router is the main methodological contribution. It learns when a \emph{fixed deployed SLM} should defer to the teacher based on residual step risk, rather than predicting the difficulty of the initial query.

To make this routing problem well-defined,  we first train a stable local SLM with an offline distillation pipeline. The SLM is warm-started with behavioral cloning on successful teacher trajectories, including their perturbed variants. It is then refined with verifier-guided DPO and consistency regularization across perturbation instantiations, while keeping the BC-trained policy frozen as the DPO reference. We use this setup because it is stable and reproducible, not because BC or DPO are themselves new optimization contributions. Only after the distilled SLM policy is fixed do we train the router on that policy's residual failures. This separation ensures that the main experiments isolate compute allocation rather than entangling routing with changes to the local policy.


The router is trained as a calibrated probabilistic decision interface. Brier regularization encourages $r_\psi(f_t)$ to estimate the conditional probability that continuing with the SLM will fail, while a CVaR-constrained objective discourages routers that are cheap on average but brittle under tail perturbation seeds \citep{chow_cvar_2014, tamar_cvar_2015}. The resulting deployment rule is simple: execute locally by default, but escalate when calibrated residual risk exceeds the routing threshold and budget remains.

Empirically, R2V improves the reliability-cost frontier across HumanEval+, TextWorld, and TerminalBench. Averaged over Gemma-9B, LLaMA-3.1-8B, Qwen2.5-7B, and Qwen2.5-14B backbones, R2V reaches $94.3\%$ success on HumanEval+ with only $0.6\%$ LLM escalation; recovers TextWorld to the $98.2\%$ LLM success ceiling at $41.7\%$ escalation, close to the $35.4\%$ oracle rate; and improves TerminalBench success from $81.1\%$ to $93.3\%$ while using $33.9\%$ LLM calls, compared with $66.0\%$ for the heuristic router at comparable reliability.

\begin{tcolorbox}[colback=LightGreenBox,
leftrule=0.5mm,top=0.5mm,bottom=0.5mm]

\textbf{Our contributions are three fold:}
\begin{itemize}
    \item We formulate SLM-LLM agent routing as a \emph{stateful step-level} allocation problem in perturbed POMDPs.
    \item We introduce a residual-risk routing objective that combines Brier-calibrated probability estimation with CVaR-constrained tail-risk control.
    \item We evaluate this router on top of a fixed, standard BC$\rightarrow$DPO recovery-distilled SLM across three benchmarks and four backbones, showing that calibrated step routing substantially reduces frontier-model calls while preserving high task success.
\end{itemize}

\vspace{-0.8em}
\end{tcolorbox}

\section{Related Work}
\label{sec:related_work}
R2V-Agent builds on efficient model routing, process supervision, recovery-oriented policy improvement, and risk-sensitive decision making; Appendix~\ref{app:related_work} provides a more extensive review. Existing LLM cascades and routers reduce inference cost by selecting which model should answer a request, either before generation or after a cheaper model has produced a response \citep{frugalgpt_2023, dekoninck_unified_2024, routellm_2024, ding_hybrid_2024, zhuang_embedllm_2024, hu_routerbench_2024}. These methods are effective for static request-level allocation, but they treat difficulty as a property of the initial prompt or completed response. R2V targets a different setting: in interactive agents, risk evolves after tool calls, partial observations, prompt perturbations, and earlier local actions. We therefore route at the step level, conditioning escalation on the unfolding agent state and verifier-grounded evidence about the deployed SLM's residual failures.

\begin{table}[H]
\caption{Comparison of R2V-Agent with representative LLM routing, cascade, verifier, and agent-training frameworks. R2V is distinguished by stateful step-level routing, perturbation recovery, and calibrated tail-risk control.}
\label{tab:comparison}
\centering
\scriptsize
\setlength{\tabcolsep}{3.2pt}
\renewcommand{\arraystretch}{0.92}
\begin{adjustbox}{max width=\textwidth}
\begin{tabular}{l c c c c c c}
\toprule
\textbf{Method}
  & \textbf{\shortstack{Step-Level\\Routing}}
  & \textbf{\shortstack{Stateful\\Agent}}
  & \textbf{\shortstack{Process\\Verifier}}
  & \textbf{\shortstack{Risk-Calib.\\(CVaR)}}
  & \textbf{\shortstack{Perturb.\\Robustness}}
  & \textbf{\shortstack{Recovery\\Training}} \\
\midrule
FrugalGPT~\citep{frugalgpt_2023}
  & \xmark & \xmark & \xmark & \xmark & \xmark & \xmark \\
RouteLLM~\citep{routellm_2024}
  & \xmark & \xmark & \xmark & \xmark & \xmark & \xmark \\
Hybrid LLM~\citep{ding_hybrid_2024}
  & \xmark & \xmark & \xmark & \xmark & \xmark & \xmark \\
Unified routing/cascading~\citep{dekoninck_unified_2024}
  & \xmark & \xmark & \xmark & \xmark & \xmark & \xmark \\
EmbedLLM~\citep{zhuang_embedllm_2024}
  & \xmark & \xmark & \xmark & \xmark & \xmark & \xmark \\
AgentPRM~\citep{xi_agentprm_2025}
  & \xmark & \cmark & \cmark & \xmark & \xmark & \xmark \\
SCoRe~\citep{kumar_score_2024}
  & \xmark & \xmark & \xmark & \xmark & \xmark & \cmark \\
\midrule
\rowcolor[HTML]{EAF5EA}
\textbf{R2V-Agent (ours)}
  & \cmark & \cmark & \cmark & \cmark & \cmark & \cmark \\
\bottomrule
\end{tabular}
\end{adjustbox}
\vspace{-0.5em}
\end{table}

R2V also relates to process reward models and verifiers for multi-step reasoning and agentic control \citep{wang_mathshepherd_2024, snell_testtime_2024, xi_agentprm_2025}. We use the verifier not only to score candidate actions, but also to construct recovery preferences and routing features. This connects to preference optimization and recovery-style training \citep{rafailov_dpo_2023, kumar_score_2024}; however, R2V uses the standard BC$\rightarrow$DPO scaffold only to obtain a stable fixed SLM, and trains the router afterward on that. Finally, R2V draws on calibration and risk-sensitive RL: Brier calibration gives a probabilistic residual-risk interface \citep{guo_calibration_2017}, while CVaR discourages routers that are cheap on average but brittle under perturbations \citep{chow_cvar_2014, tamar_cvar_2015}. Table~\ref{tab:comparison} summarizes the distinction: prior routers are largely query-level and stateless, while R2V combines stateful step routing, process verification, perturbation recovery, and CVaR-calibrated escalation.


\section{Preliminaries}
\label{sec:preliminaries}

We model agent-environment interaction as a Partially Observable Markov Decision Process (POMDP) \citep{sutton_barto_2018} augmented with a latent perturbation seed \(z\):
\begin{equation}
\label{eq:pomdp_setup}
\mathcal{M}_z=(\mathcal{S},\mathcal{A},\mathcal{O},\mathcal{Z},\mathcal{T},\mathcal{O}_z,R,\gamma),
\qquad
\mathcal{O}_z:\mathcal{S}\times\mathcal{A}\times\mathcal{Z}\to\Delta(\mathcal{O}),
\qquad
x_t=(G,o_{\le t},a_{<t}).
\end{equation}
Here \(z\in\mathcal{Z}\) captures unobserved stochastic factors such as tool flakiness, latency spikes, truncated logs, or prompt corruption; \(G\) is the task goal. A frontier policy \(\pi(a\mid x)\) is robust but expensive at every step, whereas a local SLM is efficient but brittle under out-of-distribution perturbations. R2V-Agent therefore factorizes execution into an efficient base policy \(\pi_\theta(a\mid x)\), a stronger teacher policy \(\pi_T(a\mid x)\), a lightweight process verifier \(V_\phi(x,a)\in[0,1]\), and a router \(r_\psi(f_t)\in[0,1]\) that decides whether to execute locally or escalate to the teacher.

R2V is trained as an \emph{ordered pipeline} rather than a joint multi-loss objective.  First, an offline trajectory pool is constructed by replaying teacher demonstrations under perturbation operators, yielding a mix of unperturbed and perturbed trajectories. Behavioral cloning (BC) on the episode-success subset of this pool produces a nominal policy \(\pi_\theta^{\mathrm{BC}}\). Second, this fixed BC policy generates candidate actions on the same trajectory contexts, which the verifier ranks into preference pairs, after which the SLM is refined with DPO and consistency regularization using the BC policy as the frozen DPO reference. Only after the distilled SLM is fixed do we train the router on its residual high-risk states. This separation makes the router calibrated to the failure modes of the policy actually deployed, and isolates compute allocation from changes in local-policy training.

\subsection{Verifier-Guided Distillation on Perturbed Topologies}
\label{sec:distillation}

Let \(\mathcal{D}_{\mathrm{exp}}=\{(x_t,a_t^\star)\}\) denote expert demonstrations collected from \(\pi_T\). Before BC training, we replay these trajectories under perturbation operators \(\mathcal{P}_k\) to construct a perturbed set \(\mathcal{D}_{\mathrm{pert}}\), that mimic tool failures, missing observations, adversarial strings, or other stochastic disruptions and then train BC on the episode-success subset of \(\mathcal{D}_{\mathrm{traj}}=\mathcal{D}_{\mathrm{exp}}\cup\mathcal{D}_{\mathrm{pert}}\):

\begin{equation}
\label{eq:bc}
\mathcal{L}_{\mathrm{BC}}(\theta)
=
-\mathbb{E}_{(x_t,a_t^\star)\sim\mathcal{D}_{\mathrm{traj}}}
\left[\log \pi_\theta(a_t^\star\mid x_t)\right].
\end{equation}
At each \(x_t'\) from \(\mathcal{D}_{\mathrm{traj}}\), the BC policy proposes
\[
a_t^{(1)},\ldots,a_t^{(K)}\sim \pi_\theta^{\mathrm{BC}}(\cdot\mid x_t'),
\]
which are scored by \(V_\phi\); when needed, we also query \(\pi_T\) for a recovery action \(a_t^T\). We then form preference pairs \((a_t^+,a_t^-)\), where \(a_t^+\) is verifier-approved or teacher-recovered and \(a_t^-\) is a lower-quality BC-sampled alternative. Thus, preference learning targets the BC policy's own recoverable failure modes rather than generic preference data. Appendix~\ref{sec:theory} gives auxiliary results on verifier-guided candidate selection and bounded pairwise verifier noise.

After constructing preference pairs, we freeze the BC policy as the DPO reference,
\(\pi_{\mathrm{ref}}=\operatorname{stopgrad}(\pi_\theta^{\mathrm{BC}})\), initialize from \(\pi_\theta^{\mathrm{BC}}\), and optimize DPO on perturbed states:
\begin{equation}
\label{eq:dpo}
\mathcal{L}_{\mathrm{DPO}}(\theta)
=
-\mathbb{E}_{(x_t',a^+,a^-)}
\left[
\log\sigma\left(
\beta\left[
\log\frac{\pi_\theta(a^+\mid x_t')}{\pi_{\mathrm{ref}}(a^+\mid x_t')}
-
\log\frac{\pi_\theta(a^-\mid x_t')}{\pi_{\mathrm{ref}}(a^-\mid x_t')}
\right]\right)
\right].
\end{equation}
This offline objective shifts the SLM toward perturbation-recovery actions without unstable multi-turn RL. To further stabilize behavior under partial observability, we regularize action distributions across different perturbation realizations of the same latent state:
\begin{equation}
\label{eq:consistency}
\mathcal{L}_{\mathrm{cons}}(\theta)
=
\mathbb{E}_{t,z,z'}
\left[
\operatorname{JSD}
\left(
\pi_\theta(\cdot\mid x_t^{(z)})
\;\middle\|\;
\pi_\theta(\cdot\mid x_t^{(z')})
\right)
\right].
\end{equation}
This encourages perturbation-invariant action distributions \citep{lin_jsd_1991, ppcl_2024, zhang_consistency_2023}; Appendix~\ref{sec:theory} shows that it controls the transfer gap between seen and unseen perturbation seeds.

\section{Risk-Calibrated Routing \& Verifier-Distillation (R2V)}
\label{sec:methodology}


\begin{figure}[t]
    \centering
    \includegraphics[width=0.85\textwidth]{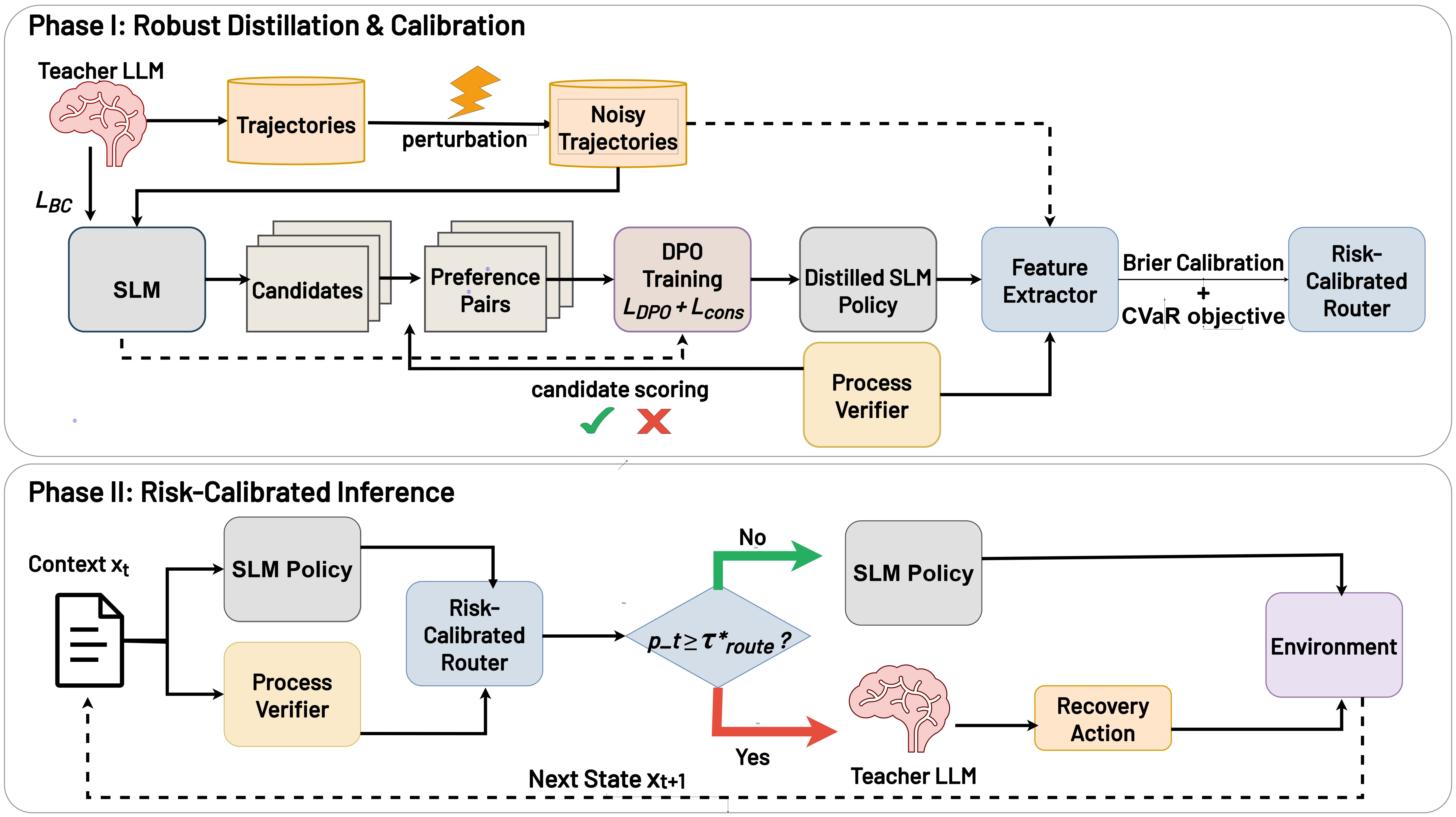}
    \caption{\textbf{R2V-Agent pipeline.}
    \textit{Phase~I:} Teacher trajectories are perturbed to train a BC-initialized SLM with verifier-guided DPO and consistency regularization; verifier and policy features then train a Brier-calibrated and CVaR-calibrated router.
    \textit{Phase~II:} At inference, the SLM acts by default, while the teacher LLM is invoked only when the router's residual-risk estimate exceeds $\tau^{*}_{\mathrm{route}}$.}
    \label{fig:r2v_method}
\end{figure}

The fixed local policy used by the routing experiments is obtained in two explicit stages:
\begin{align}
\label{eq:stage_bc}
\theta_{\mathrm{BC}}
&=
\arg\min_\theta \mathcal{L}_{\mathrm{BC}}(\theta),
\end{align}
\begin{align}
\label{eq:stage_recovery}
\theta^\star
&=
\arg\min_{\theta:\,\theta_0=\theta_{\mathrm{BC}}}
\left[
\mathcal{L}_{\mathrm{DPO}}(\theta;\pi_{\mathrm{ref}})
+
\lambda_{\mathrm{cons}}\mathcal{L}_{\mathrm{cons}}(\theta)
\right],
\qquad
\pi_{\mathrm{ref}}
=
\operatorname{stopgrad}(\pi_{\theta_{\mathrm{BC}}}).
\end{align}
Equations~\eqref{eq:stage_bc}--\eqref{eq:stage_recovery} specify the standard offline recovery-distillation scaffold used to produce a stable SLM. BC provides the nominal policy and frozen reference distribution; verifier-guided DPO and consistency provide perturbation-aware refinement. The core R2V contribution begins after \(\pi_\theta=\pi_{\theta^\star}\) is fixed: a calibrated router estimates the SLM's residual failure risk and allocates teacher compute only when intervention is warranted.

\subsection{Risk-Calibrated Trajectory Routing via CVaR}
\label{sec:routing}

For each context \(x_t\), the fixed SLM proposes \(K\) candidate actions, the verifier scores them, and a compact risk feature vector is extracted:
\begin{equation}
\label{eq:router_features}
f_t=\Phi(x_t,\pi_\theta,V_\phi)\in\mathbb{R}^{15}.
\end{equation}
The features include token-level entropy and log-probability statistics; verifier-score mean, standard deviation, spread, best score, and worst score over candidates; candidate consistency and semantic entropy; horizon fraction and absolute step index; normalized context length; and a goal-length proxy. Benchmark and perturbation-type one-hot features are implemented but disabled in the reported experiments.

The routing label is derived from the downstream SLM episode outcome, not from counterfactual teacher success. For a logged SLM rollout \(\tau^{\mathrm{SLM}}(z)\) under perturbation seed \(z\), let \(S_{\mathrm{SLM}}(z)\in\{0,1\}\) denote binary episode success. Each visited context receives
\begin{equation}
\label{eq:router_label}
y_t=\mathbf{1}\!\left[S_{\mathrm{SLM}}(z)=0\right].
\end{equation}
Thus, \(y_t\) is not a ground-truth step-correctness label; it marks contexts on SLM-failure trajectories. The router \(r_\psi(f_t)\) estimates the residual probability that continuing with the deployed SLM will fail the episode from the current local evidence.

For smooth offline training, the router probability
\begin{equation}
\label{eq:soft_router_decision}
p_t=r_\psi(f_t)
\end{equation}
relaxes the hard escalation decision. The cost-sensitive surrogate and Brier calibration loss are
\begin{align}
\label{eq:routing_loss}
\ell_{\mathrm{route}}(p_t,y_t)
&=
c_{\mathrm{SLM}}(1-p_t)+c_{\mathrm{LLM}}p_t+\kappa y_t(1-p_t),
\qquad c_{\mathrm{SLM}}<c_{\mathrm{LLM}},\;\kappa>0,
\\
\label{eq:brier}
\mathcal{L}_{\mathrm{Brier}}(\psi)
&=
\mathbb{E}_{(f_t,y_t)\sim\mathcal{D}_{\mathrm{route}}}
\left[
\left(r_\psi(f_t)-y_t\right)^2
\right].
\end{align}
To control tail failures across perturbation seeds, we aggregate stepwise routing losses into a seed-level risk
\begin{equation}
\label{eq:seed_risk}
\widetilde{R}_\psi(z)
=
\frac{1}{|\tau^{\mathrm{SLM}}(z)|}
\sum_{t\in\tau^{\mathrm{SLM}}(z)}
\ell_{\mathrm{route}}\!\left(r_\psi(f_t),y_t\right),
\end{equation}
and train the router with the differentiable CVaR-constrained Lagrangian
\begin{equation}
\label{eq:cvar_router}
\min_{\psi}\max_{\lambda\ge0}
\;
\mathbb{E}_z\!\left[\widetilde{R}_\psi(z)\right]
+
\lambda\left(
\operatorname{CVaR}_{\alpha}\!\left(\widetilde{R}_\psi(z)\right)-\epsilon
\right)
+
\lambda_B\mathcal{L}_{\mathrm{Brier}}(\psi).
\end{equation}
The hard rule \(d_t=\mathbf{1}[r_\psi(f_t)\ge\tau_{\mathrm{route}}]\) is used only for validation-time threshold selection and test-time execution in Algorithm~\ref{alg:r2v_inference}; we do not differentiate through the threshold, environment transition, or remaining-budget gate. Optimization is therefore offline and smooth, while evaluation uses the actual thresholded routed policy.

Appendix~\ref{sec:theory} shows that Brier minimization recovers the conditional episode-failure risk
\begin{equation}
\label{eq:conditional_risk}
q^\star(f)=\mathbb{P}(y_t=1\mid f_t=f).
\end{equation}
For the unconstrained one-step surrogate in Eq.~\eqref{eq:routing_loss}, the Bayes threshold before clipping is
\[
\bar{\tau}_{\mathrm{route}}
=
\frac{c_{\mathrm{LLM}}-c_{\mathrm{SLM}}}{\kappa}.
\]
Since router probabilities lie in \([0,1]\), the feasible threshold is
\begin{equation}
\label{eq:threshold}
\tau_{\mathrm{route}}^\star
=
\min\left\{
1,
\max\left\{
0,
\frac{c_{\mathrm{LLM}}-c_{\mathrm{SLM}}}{\kappa}
\right\}
\right\}.
\end{equation}
When \(0\le c_{\mathrm{LLM}}-c_{\mathrm{SLM}}\le\kappa\), this reduces to the unclipped threshold. If \(c_{\mathrm{LLM}}-c_{\mathrm{SLM}}>\kappa\), escalation is too expensive for any calibrated risk below one; if \(c_{\mathrm{LLM}}\le c_{\mathrm{SLM}}\), the surrogate always prefers escalation. Algorithm~\ref{alg:r2v_inference} additionally enforces a finite LLM budget, so the remaining-budget check acts as a feasibility projection during online execution.

For the plug-in hard router \(\widehat{d}_t=\mathbf{1}[r_\psi(f_t)\ge\tau_{\mathrm{route}}^\star]\), the excess one-step routing cost is bounded by calibration error:
\begin{equation}
\label{eq:calibration_regret}
\mathbb{E}
\left[
\ell_{\mathrm{route}}(\widehat{d}_t,y_t)
-
\ell_{\mathrm{route}}(d_t^\star,y_t)
\right]
\le
\kappa
\mathbb{E}
\left[
\left|r_\psi(f_t)-q^\star(f_t)\right|
\right].
\end{equation}
This makes calibrated residual-risk probabilities, rather than ad hoc entropy thresholds, the central routing interface.

\subsection{Inference with Risk-Calibrated Escalation}
\label{sec:inference}

At test time, the distilled SLM executes by default, while the teacher LLM is reserved for steps whose calibrated residual risk exceeds the routing threshold and for which budget remains. Algorithm~\ref{alg:r2v_inference} summarizes the full procedure: the SLM proposes actions, the verifier scores local quality, and the router allocates expensive teacher compute only at predicted high-risk steps.

\section{Experiments}
\label{sec:experiments}

We evaluate whether R2V-Agent preserves frontier-model reliability while invoking the LLM only on high-risk agent steps. Following the step-level execution interface in Figure~\ref{fig:r2v_overview}, Figure~\ref{fig:r2v_method}, and Algorithm~\ref{alg:r2v_inference}, the SLM proposes candidate actions, the process verifier scores local action quality, and the router escalates only when the predicted residual risk exceeds the calibrated threshold.

\subsection{Experimental Setup}
\label{sec:exp_setup}

\paragraph{Benchmarks.}
We evaluate R2V on three sequential decision-making benchmarks that stress different forms of agent risk: code generation, text-based embodied control, and terminal engineering. All benchmarks use a task-level $70/15/15$ train/validation/test split with seed $42$, and each clean trajectory is replayed under $5$ independently sampled perturbation seeds. On \textbf{HumanEval+} \citep{liu_evalplus_2023}, we use all $164$ Python programming tasks and cast each problem as a multi-step environment with \texttt{write\_code}, \texttt{test}, and \texttt{submit} actions; Gemini~3~Flash~Preview provides $820$ clean teacher trajectories. On \textbf{TextWorld}, we use ALFWorld \citep{shridhar_alfworld_2020}, a household-control benchmark with $250$ tasks across six task families, where agents issue textual commands such as \emph{go to}, \emph{take}, \emph{put}, \emph{clean}, and \emph{heat}; Gemini~3~Flash~Preview provides $800/200/250$ train/validation/test episodes. HumanEval+ and TextWorld use four perturbation families: tool flakiness, partial observability, prompt injection, and distractor observations. On \textbf{TerminalBench} \citep{lee_terminalbench_2025}, we use $89$ real-world shell tasks from \texttt{yoonholee/terminalbench-trajectories} \citep{yoonholee_terminalbench_trajectories}, including toolchain compilation, git-history repair, gRPC implementation, and numerical-kernel optimization. TerminalBench teacher trajectories are collected from Claude~Opus~4.6, GPT-5, and Gemini~3.1~Pro, and its perturbations are restricted to tool flakiness and partial observability because semantic prompt injection and distractors are less meaningful for shell execution.

\paragraph{Verifiers and router.}
Each benchmark uses a lightweight CPU-only process verifier $V_\phi(x_t,a_t)\in[0,1]$ that provides local evidence rather than oracle supervision. HumanEval+ combines smoke-test execution with structural code checks; TextWorld scores observation quality, action type, goal alignment, and repetition; and TerminalBench parses terminal output for success/failure signatures, command category, repetition, and completion intent. Router labels are derived from downstream SLM-only episode outcomes: a visited context receives $y_t=1$ if the fixed SLM rollout containing that context fails the task, and $y_t=0$ otherwise. Thus, labels mark residual SLM failure risk rather than counterfactual teacher correctness at each step. Full verifier definitions and label statistics are given in Appendix~\ref{app:exp_details}.

The router $r_\psi:\mathbb{R}^{15}\rightarrow[0,1]$ is a two-layer MLP with hidden widths $128$ and $64$, GELU activations, dropout $p=0.2$, Brier calibration, and post-hoc temperature scaling \citep{guo_calibration_2017}. Its inputs summarize SLM uncertainty, top-$K$ candidate log-probabilities, verifier-score statistics, candidate diversity, horizon position, and context length. We use $K=5$ candidates sampled with vLLM \citep{kwon_vllm_2023}; unless otherwise stated, R2V uses $\alpha=0.20$, $\varepsilon=0.10$, and $c_{\mathrm{LLM}}/c_{\mathrm{SLM}}=50$.

\paragraph{Baselines.}
We compare against \textbf{SLM-only}, \textbf{LLM-only}, \textbf{Entropy Router}, \textbf{Heuristic Router}, and an offline \textbf{Oracle Router}. The entropy router escalates using a validation-calibrated token-entropy threshold, while the heuristic router uses non-learned verifier rules. The oracle is a diagnostic upper bound: it has hindsight access to the SLM-only episode outcome and marks contexts on failed SLM rollouts before execution. It is therefore not an attainable policy under Algorithm~\ref{alg:r2v_inference}, since an online router must infer risk from $f_t$ without knowing the future trajectory. All deployable methods share the same distilled SLM and verifier and differ only in the routing decision, so the main results in Table~\ref{tab:main_results} evaluate step-level compute allocation rather than attributing gains to the standard BC/DPO distillation scaffold.

\subsection{Main Results}
\label{sec:main_results}

Table~\ref{tab:main_results} reports success rate (SR) and LLM escalation rate averaged across four SLM backbones: Gemma-9B \citep{gemma_2024}, LLaMA-3.1-8B \citep{llama3_2024}, Qwen2.5-7B, and Qwen2.5-14B \citep{qwen25_2024}. Since all deployable methods share the same distilled SLM and verifier, these results isolate the routing decision. The theory in Section~4 predicts that a calibrated router should escalate only when the estimated residual failure risk justifies the teacher call, while the CVaR constraint should avoid policies that are cheap on average but brittle under perturbation seeds. The results match this behavior: R2V uses little teacher compute in low-failure regimes, escalates substantially when hidden state risk is high, and remains below heuristic-router escalation on the most heterogeneous benchmark.


\begin{table}[t]
\centering
\caption{\textbf{Benchmark-level results} averaged over four SLM backbones under noisy evaluation. ``LLM\%'' is the fraction of steps escalated to the teacher. R2V uses $\alpha{=}0.20$, $\varepsilon{=}0.10$. The oracle is a non-deployable hindsight diagnostic.}
\label{tab:main_results}
\setlength{\tabcolsep}{3.5pt}
\begin{tabular}{lcccccc}
\toprule
\textbf{Method} &
\multicolumn{2}{c}{\textbf{HumanEval+}} &
\multicolumn{2}{c}{\textbf{TextWorld}} &
\multicolumn{2}{c}{\textbf{TerminalBench}} \\
\cmidrule(lr){2-3}\cmidrule(lr){4-5}\cmidrule(lr){6-7}
& SR (\%) & LLM\% & SR (\%) & LLM\% & SR (\%) & LLM\% \\
\midrule
SLM-only            & 91.9 & 0.0\%  & 64.6 & 0.0\%  & 81.1 & 0.0\% \\
Entropy Router      & 92.9 & 0.8\%  & 64.6 & 0.0\%  & 87.4 & 4.1\% \\
Heuristic Router    & 97.4 & 26.6\% & 98.4 & 99.9\% & 95.0 & 66.0\% \\
\rowcolor{LightGreenBox}
\textbf{R2V (ours)} & \textbf{94.3} & \textbf{0.6\%} &
                       \textbf{98.2} & \textbf{41.7\%} &
                       \textbf{93.3} & \textbf{33.9\%} \\
\rowcolor{LightBlueBox}
Oracle Router       & 98.6 & 8.1\%  & 98.6 & 35.4\% & 97.8 & 18.4\% \\
LLM-only            & 98.8 & 100\%  & 98.7 & 100\%  & 98.1 & 100\% \\
\bottomrule
\end{tabular}
\end{table}

\paragraph{HumanEval+.}
HumanEval+ is a low-failure regime: only $8.11\%$ of steps belong to failed SLM episodes. Consistent with the calibrated-threshold view, R2V rarely escalates. It improves average SR from $91.9\%$ to $94.3\%$ while using only $0.6\%$ LLM calls. The entropy router gives a smaller gain at similar cost, whereas the heuristic router reaches higher SR but requires $26.6\%$ LLM calls. Thus, on self-correcting code tasks, most of the benefit comes from identifying a small set of high-risk steps rather than broadly increasing teacher usage.

\paragraph{TextWorld.}
TextWorld is the clearest test of state-dependent risk under partial observability. The SLM-only policy reaches only $64.6\%$ SR, and the entropy router remains at $64.6\%$ with $0.0\%$ escalation, showing that token uncertainty alone misses many observation-induced failures. R2V reaches $98.2\%$ SR with $41.7\%$ LLM calls, within $0.4$ points of the hindsight oracle's $98.6\%$ SR and within $0.5$ points of LLM-only execution. The heuristic router reaches $98.4\%$ SR, but at $99.9\%$ escalation, effectively collapsing to LLM-only execution. This supports the residual-risk formulation: the useful signal is not just local entropy, but calibrated risk from verifier, uncertainty, and trajectory-state features.

\paragraph{TerminalBench.}
TerminalBench is the most heterogeneous setting. SLM-only performance is much lower at $81.1\%$ SR, and the oracle gap is larger than on the other benchmarks. R2V improves SR to $93.3\%$ while using $33.9\%$ LLM calls, roughly half the heuristic router's $66.0\%$ escalation rate. The remaining gap to the oracle, which reaches $97.8\%$ SR with $18.4\%$ LLM calls, indicates that some shell-task failures are identifiable in hindsight but not yet fully captured by the current online risk features. Per-backbone results in Appendix~\ref{app:per_model_results} show that R2V both recovers weaker models and reduces unnecessary escalation for stronger models.

\begin{figure*}[t]
\centering
\includegraphics[width=\textwidth]{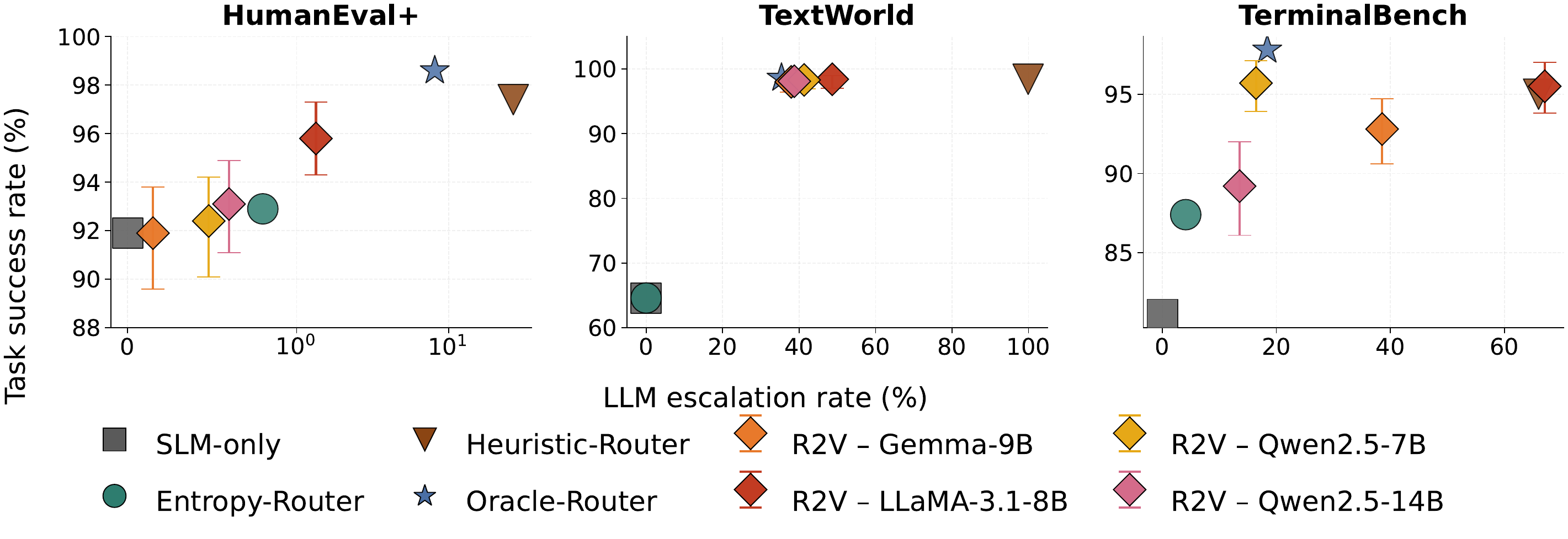}
\caption{\textbf{Cost-performance Pareto frontier.} Each R2V point corresponds to one SLM backbone with 95\% bootstrap confidence intervals. R2V gives near-free gains on HumanEval+, closely tracks the oracle on TextWorld, and recovers substantial SR for weaker TerminalBench backbones while remaining below heuristic-router cost.}
\label{fig:pareto}
\end{figure*}

Figure~\ref{fig:pareto} summarizes the same trend as a cost-performance frontier. HumanEval+ occupies a low-risk regime where escalation is rarely needed; TextWorld requires roughly $40\%$ escalation because failures are strongly tied to hidden state and observation quality; and TerminalBench exposes residual headroom, with the hindsight oracle reaching $97.8\%$ SR at only $18.4\%$ escalation. This oracle point should be interpreted as a diagnostic estimate of recoverable failure mass in logged SLM rollouts, not as a deployable policy with online access to future outcomes.

\subsection{Ablations}
\label{sec:ablations}

Table~\ref{tab:compact_ablations} summarizes targeted ablations of the components most directly tied to R2V's routing interface; full sweeps are in Appendix~\ref{app:full_ablations}. These ablations do not aim to causally decompose every SLM training stage. BC and DPO are used as a standard offline scaffold to produce a fixed deployed SLM; the main experimental question is whether a calibrated step-level router can allocate teacher compute more effectively than entropy or heuristic rules. Within this scope, the consistency rows test how perturbation regularization changes the operating point, while the feature rows test which inference-time signals are needed for risk estimation.


\begin{table}[t]
\centering
\caption{\textbf{Compact ablation summary.} Cells report SR / LLM\%. Consistency and feature ablations are evaluated on Qwen2.5-7B. Feature ablations are applied at inference time without retraining. Closed-source rows are averaged over all four backbones.}
\label{tab:compact_ablations}
\setlength{\tabcolsep}{3.2pt}
\begin{tabular}{llccc}
\toprule
\textbf{Ablation} & \textbf{Variant} &
\textbf{HumanEval+} & \textbf{TextWorld} & \textbf{TerminalBench} \\
\midrule
\rowcolor{LightBlueBox}
Consistency & $\lambda_{\mathrm{cons}}=0$   & 94.0 / 1.1 & 97.8 / 35.4 & 93.8 / 35.7 \\
\rowcolor{LightGreenBox}
\textbf{Consistency} & $\boldsymbol{\lambda_{\mathrm{cons}}=0.20}$ & \textbf{92.4 / 0.4} & \textbf{98.3 / 41.2} & \textbf{95.7 / 16.3} \\
\midrule
\rowcolor{LightBlueBox}
Features & w/o verifier & 92.4 / 1.5 & 62.6 / 19.8 & 90.6 / 2.2 \\
\rowcolor{LightBlueBox}
Features & w/o entropy  & 92.3 / 0.3 & 62.6 / 31.8 & 91.8 / 12.0 \\
\rowcolor{LightBlueBox}
Features & entropy only & 96.0 / 20.4 & 64.6 / 2.4 & 87.6 / 18.7 \\
\midrule
\rowcolor{LightBlueBox}
Closed API & no entropy      & 92.8 / 0.5 & 97.4 / 37.0 & 90.9 / 25.6 \\
\rowcolor{LightGreenBox}
\textbf{Closed API} & \textbf{pseudo-entropy}  & \textbf{94.2 / 0.6} & \textbf{98.1 / 37.4} & \textbf{92.8 / 27.1} \\
\bottomrule
\end{tabular}
\end{table}

Consistency regularization mainly changes the cost--reliability operating point. On HumanEval+, moving from $\lambda_{\mathrm{cons}}=0$ to $0.20$ lowers escalation from $1.1\%$ to $0.4\%$ but also reduces SR from $94.0\%$ to $92.4\%$. TextWorld is stable, improving slightly from $97.8\%$ to $98.3\%$ SR. TerminalBench benefits most: $\lambda_{\mathrm{cons}}=0.20$ increases SR from $93.8\%$ to $95.7\%$ while reducing LLM calls from $35.7\%$ to $16.3\%$. This is consistent with the perturbation-invariance motivation of the JSD term: the largest gains appear in the setting with high observation and tool-output variability.

The CVaR sweep in Appendix~\ref{app:full_ablations} supports the tail-risk objective. The canonical setting $(\alpha,\varepsilon)=(0.20,0.10)$ achieves $97.76\%$ SR at $27.47\%$ LLM calls. A more conservative constraint, $(0.05,0.02)$, raises SR to $98.05\%$ but increases escalation to $56.63\%$, while a relaxed setting, $(0.20,0.15)$, reduces escalation to $21.11\%$ at $95.29\%$ SR. Thus, $\varepsilon$ behaves as intended: tighter tail-risk control buys reliability through additional teacher calls.

Feature ablations show why calibrated routing needs both local quality evidence and uncertainty signals. TextWorld collapses to $62.6\%$ SR when either verifier or entropy features are removed, and entropy-only routing stays near SLM-only performance at $64.6\%$ SR. HumanEval+ is more forgiving: entropy-only routing reaches $96.0\%$ SR, but at a high $20.4\%$ escalation rate. TerminalBench also depends on richer risk features: removing verifier or entropy reduces SR to $90.6\%$ and $91.8\%$, respectively.

\subsection{Closed-Source Model Adaptation}
\label{sec:closed_source}

Commercial SLM APIs may hide token log-probabilities, so entropy-based router features may be unavailable. Table~\ref{tab:closed_source} (in Appendix~\ref{app:closed_source_app}) evaluates two black-box variants: \textbf{no-entropy}, which zeros policy-uncertainty dimensions and uses verifier, candidate-diversity, and context features; and \textbf{pseudo-entropy}, which replaces token entropy with normalized entropy over verifier scores. For scores $s_k=V_\phi(x_t,a_t^{(k)})$ over $K$ sampled candidates,
\[
\tilde{p}_k=\frac{\exp(s_k)}{\sum_{j=1}^{K}\exp(s_j)},\qquad
H_{\mathrm{ver}}(x_t)=-\frac{1}{\log K}\sum_{k=1}^{K}\tilde{p}_k\log \tilde{p}_k .
\]
This provides a black-box uncertainty proxy without SLM token log-probabilities.


R2V degrades gracefully on HumanEval+ and TextWorld under closed-source constraints. On HumanEval+, no-entropy changes SR only from $93.3\%$ to $92.8\%$, while pseudo-entropy recovers to $93.2\%$. On TextWorld, pseudo-entropy nearly matches the full router, reaching $98.1\%$ SR with ECE $0.054$ and Brier $0.106$. TerminalBench is more sensitive: no-entropy drops from $93.3\%$ to $90.9\%$ SR and worsens ECE from $0.162$ to $0.238$, while pseudo-entropy partially recovers performance and calibration, reaching $92.8\%$ SR with ECE $0.181$ and Brier $0.204$. Thus, black-box deployment is practical for code and navigation tasks, while terminal-engineering agents benefit from true log-probability access when available.

\section{Conclusion}
\label{sec:conclusion}

We introduced \textbf{R2V-Agent}
\footnote{\textbf{Code:} \href{https://github.com/RaghuHemadri/r2v-agent.git}{github.com/RaghuHemadri/r2v-agent}.}
, a step-level routing framework that helps small language model agents decide when to ask a stronger LLM for help. Instead of routing the whole task before execution, R2V estimates risk at each agent step using SLM uncertainty, process-verifier signals, and trajectory features. The SLM is trained with a standard BC$\rightarrow$DPO recovery-distillation pipeline, while the main contribution is the calibrated router trained with Brier calibration and CVaR-based tail-risk control.

Across HumanEval+, TextWorld, and TerminalBench, R2V improves the reliability--cost trade-off: it keeps most steps local, escalates only when risk is high, and uses fewer teacher calls than heuristic routing at similar reliability. These results show that efficient agents do not need to choose between cheap but brittle SLM execution and expensive full LLM execution; calibrated step-level routing provides a practical middle ground.

R2V still depends on the quality of the verifier and the risk features, and the remaining oracle gap on TerminalBench shows that some failures are not yet easy to detect online. Future work can improve the verifier, learn richer state features, and adapt the LLM budget dynamically across longer and more realistic agent tasks.



\newpage

\bibliography{refrences}
\bibliographystyle{colm2026_conference}

\newpage

\appendix


\section{Algorithms}
\label{sec:algo}

\begin{algorithm}[H]
\caption{R2V-Agent Training Pipeline}
\label{alg:r2v_training}
\begin{algorithmic}[1]
\REQUIRE Teacher LLM \(\pi_T\), initial SLM \(\pi_\theta\), verifier \(V_\phi\), perturbation operators \(\{\mathcal{P}_k\}\)
\STATE Collect expert trajectories \(\mathcal{D}_{\mathrm{exp}}\) using \(\pi_T\) on training tasks
\STATE Apply \(\{\mathcal{P}_k\}\) across seeds \(z \in \mathcal{Z}\) to obtain perturbed trajectories \(\mathcal{D}_{\mathrm{pert}}\) and form offline trajectory pool \(\mathcal{D}_{\mathrm{traj}} \leftarrow \mathcal{D}_{\mathrm{exp}} \cup \mathcal{D}_{\mathrm{pert}}\)
\STATE Warm-start the SLM by minimizing \(\mathcal{L}_{\mathrm{BC}}\) on \(\mathcal{D}_{\mathrm{traj}}\), yielding \(\pi_\theta^{\mathrm{BC}}\)
\STATE Initialize empty datasets \(\mathcal{D}_{\mathrm{pref}}, \mathcal{D}_{\mathrm{cons}}, \mathcal{D}_{\mathrm{route}}\)
\FOR{contexts \(x_t'\) from \(\mathcal{D}_{\mathrm{traj}}\)}
    \STATE Sample \(K\) candidate actions \(a_t^{(1:K)} \sim \pi_\theta^{\mathrm{BC}}(\cdot \mid x_t')\)
    \STATE Score candidates with \(V_\phi(x_t', a_t^{(k)})\); optionally query teacher action \(a_t^T \sim \pi_T(\cdot \mid x_t')\)
    \STATE Construct preference pairs \((x_t', a_t^+, a_t^-)\) and add them to \(\mathcal{D}_{\mathrm{pref}}\)
    \STATE Construct paired perturbation views for consistency regularization and add them to \(\mathcal{D}_{\mathrm{cons}}\)
\ENDFOR
\STATE Freeze a DPO reference policy \(\pi_{\mathrm{ref}} \leftarrow \operatorname{stopgrad}(\pi_\theta^{\mathrm{BC}})\)
\STATE Continue training the SLM from \(\pi_\theta^{\mathrm{BC}}\) using only \(\beta_1 \mathcal{L}_{\mathrm{DPO}} + \beta_2 \mathcal{L}_{\mathrm{cons}}\), yielding distilled policy \(\pi_\theta\)
\FOR{rollouts of the distilled policy \(\pi_\theta\) under diverse perturbation seeds}
    \STATE At each step, sample SLM candidates, score them with \(V_\phi\), and extract risk features \(f_t\)
    \STATE Define episode-failure label \(y_t\) from the SLM-only continuation outcome under perturbation seed \(z\)
    \STATE Add \((f_t, y_t)\) to \(\mathcal{D}_{\mathrm{route}}\)
\ENDFOR
\STATE Train router \(r_\psi\) on \(\mathcal{D}_{\mathrm{route}}\) with Brier calibration and a CVaR-constrained cost objective
\RETURN Distilled SLM \(\pi_\theta\), verifier \(V_\phi\), router \(r_\psi\)
\end{algorithmic}
\end{algorithm}

\begin{algorithm}[H]
\caption{R2V-Agent Risk-Calibrated Inference Loop}
\label{alg:r2v_inference}
\begin{algorithmic}[1]
\REQUIRE Distilled SLM \(\pi_\theta\), Teacher LLM \(\pi_T\), Verifier \(V_\phi\), Router \(r_\psi\), Threshold \(\tau_{\mathrm{route}}^\star\)
\STATE Initialize context \(x_0 = (G, o_0)\)
\FOR{$t = 0, 1, \dots, H-1$}
    \STATE Sample \(K\) candidate actions \(a_t^{(1)}, \dots, a_t^{(K)} \sim \pi_\theta(\cdot \mid x_t)\)
    \STATE Score candidates \(s_t^{(k)} = V_\phi(x_t, a_t^{(k)})\) and set \(\hat{a}_t^{\mathrm{SLM}} = \arg\max_k s_t^{(k)}\)
    \STATE Compute risk features \(f_t = \Phi(x_t,\pi_\theta,V_\phi)\)

    \STATE Compute escalation probability \(p_t=r_\psi(f_t)\)
    \STATE Set hard deployment decision \(d_t=\mathbf{1}[p_t\ge \tau_{\mathrm{route}}^\star]\)
    \IF{\(d_t=1\) \AND Remaining Budget \(> 0\)}

        \STATE Sample \(a_t \sim \pi_T(\cdot \mid x_t)\)
        \STATE Decrement LLM Budget
    \ELSE
        \STATE Set \(a_t = \hat{a}_t^{\mathrm{SLM}}\)
    \ENDIF
    \STATE Execute \(a_t\), observe \(o_{t+1} \sim \mathcal{O}_z(\cdot \mid x_t,a_t,z)\)
    \STATE Update context \(x_{t+1} = (x_t, a_t, o_{t+1})\)
\ENDFOR
\end{algorithmic}
\end{algorithm}

\section{Theoretical Guarantees for Risk-Calibrated Routing}
\label{sec:theory}

This appendix collects the formal assumptions and proofs omitted from the main text.

We study the agent defined on the noisy-observation POMDP
\[
\mathcal{M}_z = (\mathcal{S}, \mathcal{A}, \mathcal{O}, \mathcal{Z}, \mathcal{T}, \mathcal{O}_z, R, \gamma),
\]
where \(z \in \mathcal{Z}\) denotes the perturbation seed, \(x_t\) the step-\(t\) context, \(\pi_\theta\) the distilled SLM, \(\pi_T\) the teacher LLM, \(V_\phi\) the verifier, and \(r_\psi\) the router.

At each step \(t\), the SLM generates \(K\) candidates
\[
a_t^{(1)},\dots,a_t^{(K)} \sim \pi_\theta(\cdot \mid x_t),
\]
the verifier scores them by
\[
s_t^{(k)} = V_\phi(x_t, a_t^{(k)}),
\]
and the best SLM candidate is
\[
\hat a_t^{\mathrm{SLM}}
=
\operatorname*{arg\,max}_{1 \le k \le K}
V_\phi(x_t, a_t^{(k)}).
\]

The router receives a feature vector \(f_t\) and outputs
\[
p_t = r_\psi(f_t)\in[0,1].
\]
A hard routing decision is then obtained by thresholding:
\[
d_t = \mathbf{1}[p_t \ge \tau_{\mathrm{route}}].
\]


Throughout, \(y_t\in\{0,1\}\) denotes the episode-derived SLM failure label attached to step \(t\). For a logged SLM rollout \(\tau^{\mathrm{SLM}}(z)\), let \(S_{\mathrm{SLM}}(z)\in\{0,1\}\) be its binary episode success. Each context \(x_t\in\tau^{\mathrm{SLM}}(z)\) receives \(y_t=\mathbf{1}[S_{\mathrm{SLM}}(z)=0]\). This label is not a per-step correctness annotation and is not defined by comparing the verifier score of the SLM action against a teacher action. The teacher enters the framework as the escalation policy used by the routed system; the calibration target for the router is the conditional probability that the deployed SLM will fail the episode from the current local evidence.

\subsection{Assumptions}

\begin{assumption}[Boundedness and measurability]
\label{ass:bounded}
For every step \(t\),
\[
V_\phi(x_t,a_t) \in [0,1],
\qquad
r_\psi(f_t) \in [0,1],
\qquad
y_t \in \{0,1\},
\qquad
S_z \in \{0,1\}.
\]
Moreover, the feature vector satisfies
\[
\|f_t\|_2 \le B_f
\]
for some finite constant \(B_f > 0\).
\end{assumption}

\begin{assumption}[Well-defined calibration target]
\label{ass:calibration_target}
Define the conditional escalation probability
\[
q^\star(f) := \mathbb{P}(y_t = 1 \mid f_t = f).
\]
We assume that \(q^\star(f)\) is well-defined for almost every \(f\).
\end{assumption}

\begin{assumption}[Cost-sensitive routing surrogate]
\label{ass:cost_sensitive}
For hard routing analysis, we evaluate a routing decision \(d_t \in \{0,1\}\) with the stepwise surrogate loss
\[
\ell_{\mathrm{route}}(d_t,y_t)
=
c_{\mathrm{SLM}}(1-d_t)
+
c_{\mathrm{LLM}}d_t
+
\kappa\, y_t(1-d_t),
\]
where \(c_{\mathrm{SLM}} < c_{\mathrm{LLM}}\) and \(\kappa > 0\) is the penalty for not escalating on a step where escalation was actually needed.
\end{assumption}

\begin{assumption}[Perturbation pairing and stepwise Lipschitzness]
\label{ass:consistency}
Suppose two perturbation seeds \(z,z'\in\mathcal Z\) produce two observed contexts \(x_t^{(z)}\) and \(x_t^{(z')}\) corresponding to the same latent environment state \(s_t\). Let
\[
P_t^{(z)} := \pi_\theta(\cdot \mid x_t^{(z)}),
\qquad
P_t^{(z')} := \pi_\theta(\cdot \mid x_t^{(z')}).
\]
Assume the one-step surrogate loss \(\ell_t(P;s_t)\in[0,1]\) satisfies
\[
|\ell_t(P;s_t)-\ell_t(Q;s_t)|
\le
L\,\operatorname{TV}(P,Q)
\]
for all action distributions \(P,Q\) and some constant \(L>0\).
\end{assumption}

\subsection{Calibration and Bayes-optimal routing}

\begin{theorem}[Brier-optimal calibration of the router]
\label{thm:brier_calibration}
Under Assumptions~\ref{ass:bounded} and \ref{ass:calibration_target}, the population minimizer of the Brier risk
\[
\mathcal{R}_{\mathrm{Brier}}(r_\psi)
=
\mathbb{E}\bigl[(r_\psi(f_t)-y_t)^2\bigr]
\]
is the conditional escalation probability
\[
r^\star(f) = q^\star(f) = \mathbb{P}(y_t=1 \mid f_t=f).
\]
\end{theorem}

\begin{proof}
Condition on \(f_t=f\), and write \(p=r_\psi(f)\) and \(q=q^\star(f)=\mathbb{P}(y_t=1\mid f_t=f)\). Since \(y_t\in\{0,1\}\),
\[
\mathbb{E}\bigl[(p-y_t)^2 \mid f_t=f\bigr]
=
q(p-1)^2 + (1-q)p^2
=
(p-q)^2 + q(1-q).
\]
The second term is constant in \(p\), and the first is uniquely minimized at \(p=q\). Taking expectation over \(f_t\) proves the claim.
\end{proof}


\begin{corollary}[Cost-optimal hard routing threshold]
\label{cor:threshold}
Under Assumptions~\ref{ass:bounded}, \ref{ass:calibration_target}, and \ref{ass:cost_sensitive}, the Bayes-optimal hard routing rule for the unconstrained one-step surrogate \(\ell_{\mathrm{route}}\) is
\[
d_t^\star
=
\mathbf{1}
\left[
q^\star(f_t)
\ge
\tau_{\mathrm{route}}^\star
\right],
\]
where
\[
\tau_{\mathrm{route}}^\star
=
\min\left\{1,\max\left\{0,\frac{c_{\mathrm{LLM}}-c_{\mathrm{SLM}}}{\kappa}\right\}\right\}.
\]
If \(0\le c_{\mathrm{LLM}}-c_{\mathrm{SLM}}\le \kappa\), the clipped threshold equals \((c_{\mathrm{LLM}}-c_{\mathrm{SLM}})/\kappa\).
\end{corollary}

\begin{proof}
Fix \(f_t=f\), and let \(q=q^\star(f)\). If \(d_t=1\), then
\[
\mathbb{E}\!\left[\ell_{\mathrm{route}}(1,y_t)\mid f_t=f\right]
=
c_{\mathrm{LLM}}.
\]
If \(d_t=0\), then
\[
\mathbb{E}\!\left[\ell_{\mathrm{route}}(0,y_t)\mid f_t=f\right]
=
c_{\mathrm{SLM}}+\kappa q.
\]
Escalation is optimal iff \(c_{\mathrm{LLM}}\le c_{\mathrm{SLM}}+\kappa q\), equivalently
\[
q\ge \frac{c_{\mathrm{LLM}}-c_{\mathrm{SLM}}}{\kappa}.
\]
Because \(q\in[0,1]\), thresholds outside \([0,1]\) are equivalent to the nearest feasible endpoint, yielding the clipped expression.
\end{proof}


\begin{theorem}[Regret bound from router miscalibration] 
\label{thm:regret_bound} 
Let

\[
\hat d_t
=
\mathbf{1}
\left[
r_\psi(f_t)
\ge
\tau_{\mathrm{route}}^\star
\right],
\qquad
\tau_{\mathrm{route}}^\star
=
\min\left\{1,\max\left\{0,\frac{c_{\mathrm{LLM}}-c_{\mathrm{SLM}}}{\kappa}\right\}\right\},
\]

be the plug-in hard router induced by a probabilistic router \(r_\psi\), and let \(d_t^\star\) denote the Bayes-optimal rule from Corollary~\ref{cor:threshold}. Then
\[
\mathbb{E}\!\left[
\ell_{\mathrm{route}}(\hat d_t,y_t)
-
\ell_{\mathrm{route}}(d_t^\star,y_t)
\right]
\le
\kappa\,
\mathbb{E}\!\left[
|r_\psi(f_t)-q^\star(f_t)|
\right].
\]
\end{theorem}

\begin{proof}
Condition on \(f_t=f\), and write
\[
q = q^\star(f),
\qquad
\widehat q = r_\psi(f),
\qquad
\tau^\star = \tau_{\mathrm{route}}^\star
\]
From the proof of Corollary~\ref{cor:threshold},
\[
R_f(1)=c_{\mathrm{LLM}},
\qquad
R_f(0)=c_{\mathrm{SLM}}+\kappa q,
\]
so
\[
|R_f(0)-R_f(1)| = \kappa |q-\tau^\star|.
\]
If \(\hat d(f)\neq d^\star(f)\), then \(q\) and \(\widehat q\) lie on opposite sides of \(\tau^\star\), hence
\[
|q-\tau^\star| \le |\widehat q-q|.
\]
Therefore the conditional regret is bounded by
\[
\mathbb{E}\!\left[
\ell_{\mathrm{route}}(\hat d_t,y_t)-\ell_{\mathrm{route}}(d_t^\star,y_t)
\mid f_t=f
\right]
\le
\kappa |\widehat q-q|.
\]
Taking expectation over \(f_t\) yields the result.
\end{proof}

\subsection{Consistency regularization and perturbation transfer}

\begin{lemma}[Total variation is controlled by Jensen--Shannon divergence]
\label{lem:tv_jsd}
For any two distributions \(P\) and \(Q\),
\[
\operatorname{TV}(P,Q)
\le
\sqrt{2\,\operatorname{JSD}(P\|Q)}.
\]
\end{lemma}

\begin{proof}
Let \(M=(P+Q)/2\). By the triangle inequality,
\[
\operatorname{TV}(P,Q)
\le
\operatorname{TV}(P,M)+\operatorname{TV}(Q,M).
\]
Applying Pinsker's inequality to each term gives
\[
\operatorname{TV}(P,M)\le \sqrt{\tfrac{1}{2}\operatorname{KL}(P\|M)},
\qquad
\operatorname{TV}(Q,M)\le \sqrt{\tfrac{1}{2}\operatorname{KL}(Q\|M)}.
\]
Using \(\sqrt{a}+\sqrt{b}\le \sqrt{2(a+b)}\) for \(a,b\ge0\), we obtain
\[
\operatorname{TV}(P,Q)
\le
\sqrt{
2\left(
\tfrac{1}{2}\operatorname{KL}(P\|M)
+
\tfrac{1}{2}\operatorname{KL}(Q\|M)
\right)
}
=
\sqrt{2\,\operatorname{JSD}(P\|Q)}.
\]
\end{proof}

\begin{theorem}[Perturbation-generalization bound from consistency regularization]
\label{thm:consistency_bound}
Assume Assumptions~\ref{ass:bounded} and \ref{ass:consistency}. For a perturbation seed \(z\), define the episodic surrogate risk
\[
R_z(\pi_\theta)
:=
\mathbb{E}\!\left[
\sum_{t=1}^{H}
\ell_t\!\left(\pi_\theta(\cdot\mid x_t^{(z)}); s_t\right)
\right].
\]
Then for any two perturbation seeds \(z,z'\in\mathcal Z\),
\[
|R_z(\pi_\theta)-R_{z'}(\pi_\theta)|
\le
HL\sqrt{2\,\mathcal{L}_{\mathrm{cons}}(\theta)},
\]
where
\[
\mathcal{L}_{\mathrm{cons}}(\theta)
=
\mathbb{E}_{t,z,z'}
\left[
\operatorname{JSD}\!\left(
\pi_\theta(\cdot\mid x_t^{(z)})
\;\middle\|\;
\pi_\theta(\cdot\mid x_t^{(z')})
\right)
\right].
\]
Consequently, if a training seed \(z_{\mathrm{train}}\) is representative of the seen perturbations, then for any unseen test seed \(z_{\mathrm{test}}\),
\[
R_{z_{\mathrm{test}}}(\pi_\theta)
\le
R_{z_{\mathrm{train}}}(\pi_\theta)
+
HL\sqrt{2\,\mathcal{L}_{\mathrm{cons}}(\theta)}.
\]
\end{theorem}

\begin{proof}
Write
\[
R_z(\pi_\theta)-R_{z'}(\pi_\theta)
=
\mathbb{E}\!\left[
\sum_{t=1}^{H}
\Bigl(
\ell_t(P_t^{(z)};s_t)-\ell_t(P_t^{(z')};s_t)
\Bigr)
\right].
\]
By Assumption~\ref{ass:consistency},
\[
\left|
\ell_t(P_t^{(z)};s_t)-\ell_t(P_t^{(z')};s_t)
\right|
\le
L\,\operatorname{TV}\!\left(P_t^{(z)},P_t^{(z')}\right).
\]
Applying Lemma~\ref{lem:tv_jsd},
\[
\operatorname{TV}\!\left(P_t^{(z)},P_t^{(z')}\right)
\le
\sqrt{
2\,\operatorname{JSD}\!\left(
P_t^{(z)}\middle\|P_t^{(z')}
\right)
}.
\]
Using Jensen's inequality over the expectation defining \(\mathcal{L}_{\mathrm{cons}}(\theta)\), each stepwise term is bounded by \(\sqrt{2\,\mathcal{L}_{\mathrm{cons}}(\theta)}\). Summing over \(H\) steps yields
\[
|R_z(\pi_\theta)-R_{z'}(\pi_\theta)|
\le
HL\sqrt{2\,\mathcal{L}_{\mathrm{cons}}(\theta)}.
\]
Setting \(z=z_{\mathrm{test}}\) and \(z'=z_{\mathrm{train}}\) gives the second claim.
\end{proof}

\begin{remark}
Theorem~\ref{thm:consistency_bound} is stated for an episodic surrogate risk that decomposes over steps. Whenever the binary episode failure \(1-S_z\) is upper bounded by such a surrogate, the same control transfers directly to the robust routing objective in the main text.
\end{remark}

\subsection{Auxiliary results for verifier-guided distillation}

\begin{definition}[Good candidate set]
For a fixed context \(x_t\) and threshold \(\gamma \in [0,1]\), define
\[
\mathcal{G}_\gamma(x_t)
:=
\left\{
a \in \mathcal A :
V_\phi(x_t,a) \ge \gamma
\right\}.
\]
Also define the SLM coverage probability
\[
\mu_\gamma(x_t)
:=
\mathbb{P}_{a\sim\pi_\theta(\cdot\mid x_t)}
\bigl(a \in \mathcal{G}_\gamma(x_t)\bigr).
\]
\end{definition}

\begin{proposition}[Best-of-\(K\) verifier selection improves with coverage]
\label{prop:best_of_k}
Fix a context \(x_t\), and suppose \(K\) candidate actions are sampled i.i.d. from \(\pi_\theta(\cdot\mid x_t)\). Assume that for every good action \(a^+ \in \mathcal{G}_\gamma(x_t)\) and every bad action \(a^- \notin \mathcal{G}_\gamma(x_t)\), the verifier misranks the pair with probability at most \(\eta_V\), i.e.,
\[
\mathbb{P}\!\left(
V_\phi(x_t,a^-)\ge V_\phi(x_t,a^+)
\right)
\le
\eta_V.
\]
Then
\[
\mathbb{P}\!\left(
\hat a_t^{\mathrm{SLM}} \in \mathcal{G}_\gamma(x_t)
\right)
\ge
1-(1-\mu_\gamma(x_t))^K - \frac{K(K-1)}{2}\eta_V.
\]
\end{proposition}

\begin{proof}
Let \(E_{\mathrm{good}}=\{\hat a_t^{\mathrm{SLM}}\in\mathcal{G}_\gamma(x_t)\}\), and write its complement as the union of two events: (i) no good candidate is sampled; or (ii) at least one good candidate is sampled, but a bad candidate is selected due to verifier misranking. The first event has probability
\[
(1-\mu_\gamma(x_t))^K
\]
by independence. For the second event, there are at most \(K(K-1)/2\) good--bad pairs among the \(K\) samples, and each can be misranked with probability at most \(\eta_V\); a union bound therefore gives
\[
\mathbb{P}(E_{\mathrm{bad}})
\le
(1-\mu_\gamma(x_t))^K
+
\frac{K(K-1)}{2}\eta_V.
\]
Taking complements yields the claim.
\end{proof}

\begin{proposition}[Pairwise sign consistency of verifier-noisy DPO]
\label{prop:noisy_dpo}
Consider a fixed context \(x_t\) and two actions \(a^+,a^-\). Let the true pairwise preference label be
\[
y^\star \in \{+1,-1\},
\]
where \(y^\star=+1\) means \(a^+\) is truly preferred to \(a^-\). Suppose the verifier produces a noisy pairwise label \(\widetilde y\) such that
\[
\mathbb{P}(\widetilde y = y^\star \mid x_t,a^+,a^-) = 1-\eta,
\qquad
\mathbb{P}(\widetilde y = -y^\star \mid x_t,a^+,a^-) = \eta,
\]
with \(\eta < 1/2\).

Let the DPO score difference be
\[
u_\theta(x_t,a^+,a^-)
=
\beta\left[
\log \frac{\pi_\theta(a^+\mid x_t)}{\pi_{\mathrm{ref}}(a^+\mid x_t)}
-
\log \frac{\pi_\theta(a^-\mid x_t)}{\pi_{\mathrm{ref}}(a^-\mid x_t)}
\right].
\]
Then the population minimizer of the noisy logistic loss
\[
\mathcal{L}_{\mathrm{DPO,noisy}}(\theta)
=
\mathbb{E}\Bigl[\log(1+\exp(-\widetilde y\,u_\theta))\Bigr]
\]
has the same sign as the true preference:
\[
\operatorname{sign}\bigl(u_\theta^\star(x_t,a^+,a^-)\bigr)=y^\star.
\]
\end{proposition}

\begin{proof}
It suffices to consider the case \(y^\star=+1\), since \(y^\star=-1\) is symmetric. Conditional on \((x_t,a^+,a^-)\), the noisy label equals \(+1\) with probability \(1-\eta\) and \(-1\) with probability \(\eta\). Hence the scalar conditional risk is
\[
R(u)
=
(1-\eta)\log(1+e^{-u})
+
\eta\log(1+e^u).
\]
Differentiating gives
\[
R'(u)
=
-\frac{1-\eta}{1+e^u}
+
\frac{\eta e^u}{1+e^u}
=
\frac{\eta e^u-(1-\eta)}{1+e^u}.
\]
The stationary point satisfies
\[
e^u = \frac{1-\eta}{\eta}.
\]
Since \(\eta<1/2\), we have \((1-\eta)/\eta>1\), so the minimizer \(u^\star=\log((1-\eta)/\eta)\) is strictly positive. Therefore \(\operatorname{sign}(u^\star)=+1=y^\star\). The symmetric argument yields the result for \(y^\star=-1\).
\end{proof}
\section{Additional Experimental Details}
\label{app:exp_details}

\subsection{Detailed Dataset Construction}
\label{app:dataset_details}

All benchmarks use a task-level $70/15/15$ train/validation/test split with seed $42$. Each clean trajectory is replayed under $5$ independently sampled perturbation seeds to produce the noisy training and evaluation distributions.

\paragraph{HumanEval+.}
We use the full EvalPlus benchmark \citep{liu_evalplus_2023}, comprising $164$ Python programming problems. Each problem is presented as a function signature with a docstring. The agent interacts with a three-action environment: \texttt{write\_code}, \texttt{test}, and \texttt{submit}. Clean teacher trajectories are collected with Gemini~3~Flash~Preview, yielding $820$ trajectories, i.e., $5$ per task. Perturbed variants are generated by applying all four perturbation operators simultaneously through $\mathcal{P}_{\mathrm{composite}}$: tool flakiness injects HTTP errors, stale cache markers, and truncated responses into test-runner output; partial observability drops or reorders log lines; prompt injection embeds adversarial directives in problem docstrings or test outputs; and distractors append plausible but incorrect code snippets and red-herring error messages to observations. The executor provides a deterministic binary success signal by running the full EvalPlus test suite.

\paragraph{TextWorld.}
We use ALFWorld \citep{shridhar_alfworld_2020}, a suite of interactive text-based household tasks. The benchmark comprises $250$ tasks split across six task types: pick-and-place, pick-clean-then-place, pick-heat-then-place, pick-cool-then-place, look-at-object-in-light, and pick-two-objects. The agent issues text commands and receives textual observations describing the current room, visible objects, and inventory. Teacher trajectories are collected with Gemini~3~Flash~Preview, yielding $800/200/250$ train/validation/test episodes. All four perturbation types are applied at training and test time.

\paragraph{TerminalBench.}
We use TerminalBench \citep{lee_terminalbench_2025}, a benchmark of $89$ real-world terminal engineering tasks sourced from the \texttt{yoonholee/terminalbench-trajectories} HuggingFace dataset. Tasks span systems-programming scenarios such as compiling language toolchains, repairing git histories, building gRPC services, and optimizing numerical kernels. Each task is graded by automated test suites embedded in the benchmark. Teacher trajectories are collected from a diverse pool of frontier models, including Claude~Opus~4.6, GPT-5.5, Gemini~3.1~Pro, and others. The split yields approximately $62/13/14$ train/validation/test tasks with $3{,}750$ episodes in total: $2{,}625$ clean and $1{,}125$ composite-perturbed. Unlike the other benchmarks, TerminalBench perturbations combine only tool flakiness and partial observability.

\subsection{Verifier Details}
\label{app:verifier_details}

\paragraph{HumanEval verifier.}
The HumanEval+ verifier decomposes each score into execution and structural signals. The primary signal, with weight $0.35$, runs the candidate code in an isolated subprocess against a smoke-test harness. The remaining weight is distributed across structural checks: required function definition ($0.20$), non-trivial control flow ($0.15$), non-trivial return statement ($0.10$), import sanity ($0.10$), and length regularity ($0.10$). Two multiplicative penalties are then applied: a reward-hacking penalty for hardcoded lookup tables, bare literal returns, and exception-swallowing patterns; and a repetition penalty for re-submitting code identical to a previously failed attempt. A $0.85\times$ discount is applied when prompt-injection markers are detected. When a previous test result is visible, the final score blends the structural estimate and observed test outcome using a $60/40$ split. On the held-out test split, the verifier achieves episode-level AUROC $0.856$ and Brier score $0.172$.

\paragraph{TextWorld verifier.}
The TextWorld verifier combines five local signals: previous observation quality ($0.25$), action-type base score ($0.25$), goal alignment ($0.25$), non-repetition ($0.15$), and non-oscillation ($0.10$). Observation quality parses the most recent environment response for progress indicators such as \emph{you take}, \emph{you put}, and \emph{you go to}, versus failure strings such as \emph{that action did not help}. Action-type priors score task-progress verbs such as \emph{take}, \emph{put}, \emph{clean}, and \emph{heat} above navigational or idle actions. Goal alignment measures token overlap between action arguments and key nouns extracted from the task goal. When a scalar reward is visible, the weighted score is blended with the observed reward using a $50/50$ split. On the held-out TextWorld split, mean-aggregated verifier scores achieve episode-level AUROC ${\approx}0.78$, indicating moderate but imperfect local discrimination. This supports our design choice to use the verifier as one input to the learned router rather than as an oracle routing rule.

\paragraph{TerminalBench verifier.}
The TerminalBench verifier operates without Docker or subprocess execution. It scores each action by parsing terminal output already present in the trajectory context. For \texttt{run [cmd]} actions, it extracts the most recent ``New Terminal Output:'' block, strips perturbation noise, and scans for explicit success or failure signals. Success patterns such as \texttt{PASSED}, ``all tests passed'', ``build success'', and ``wrote $N$ bytes'' yield terminal scores in $[0.70,1.00]$; failure patterns such as ``command not found'', \texttt{Traceback}, non-zero exit codes, and segmentation faults yield scores in $[0.08,0.18]$. Commands are also categorized as \emph{eval}, \emph{build}, \emph{run}, \emph{install}, \emph{write}, or \emph{explore}, with category-specific base scores. The final step score blends terminal evidence and command category, with evaluation commands receiving the privileged score $\mathrm{clamp}(0.15+0.85s_{\mathrm{terminal}})$. Repetition and adversarial-injection penalties are applied multiplicatively.

For \texttt{mark\_task\_complete}, the verifier searches the last $3{,}000$ characters for terminal success evidence and returns $\mathrm{clamp}(0.20+0.75s_{\mathrm{terminal}},0.20,0.95)$ when such evidence exists. A clean bash prompt yields $0.68$, no signal yields $0.55$, \texttt{image\_read} returns $0.52$, and unrecognized actions return $0.40$. On the held-out test split, the verifier achieves episode-level AUROC $0.508$ and Brier score $0.276$, reflecting near-chance separation when local step scores are mean-aggregated. This is expected: individual step scores encode local command quality rather than complete episode outcome, and the router learns to bridge this local-to-global gap.

\paragraph{Step labels.}
The verifier \(V_\phi\) scores each step using local signals, while the router label \(y_t\) is global but attached to the current context: it equals \(1\) when the fixed SLM rollout containing that context ends in episode failure, and \(0\) when the rollout succeeds. We do not observe ground-truth step success labels, and we do not define \(y_t\) by evaluating whether a teacher action would have prevented the failure. The router therefore learns a local-to-global mapping: from current SLM uncertainty, verifier statistics, and execution context to the probability that continuing with the SLM will fail the episode.

\begin{table}[h]
\centering

\caption{\textbf{Step-label distribution and verifier discrimination.} ``Successful step'' means the current context belongs to an SLM-only rollout whose binary episode outcome is success; ``failed step'' means it belongs to an SLM-only rollout whose binary episode outcome is failure.}
\label{tab:dataset_stats}
\begin{tabular}{lccccc}
\toprule
\textbf{Benchmark} & \textbf{Succ.\ steps} & \textbf{Fail.\ steps} & \textbf{Score (succ.)} & \textbf{Score (fail.)} & \textbf{AUROC (ep.)} \\
\midrule
HumanEval+          & 91.89\% & 8.11\%  & 0.620 & 0.336 & 0.856 \\
TextWorld           & 64.60\% & 35.40\% & ${\approx}0.63$ & ${\approx}0.51$ & ${\approx}0.78$ \\
TerminalBench       & 55.55\% & 44.45\% & 0.406 & 0.398 & 0.508 \\
\bottomrule
\end{tabular}
\end{table}

\subsection{Router Architecture and Training}
\label{app:router_training}




The router \(r_\psi:\mathbb{R}^{15}\to[0,1]\) is a shallow MLP with hidden widths \(128\) and \(64\), GELU activations, batch normalization, dropout \(p=0.2\), and a sigmoid output. Learnable post-hoc temperature scaling \citep{guo_calibration_2017} is applied to the output logits. The router has approximately \(10{,}000\) parameters and runs entirely on CPU.

At each step \(t\), the distilled SLM samples \(K=5\) candidate actions \(a_t^{(1)},\ldots,a_t^{(K)}\) with vLLM~\citep{kwon_vllm_2023}. The verifier scores all candidates, and the resulting \(15\)-dimensional feature vector \(f_t\) contains token-level entropy and log-probability statistics; verifier score mean, standard deviation, spread, best score, and worst score; candidate consistency and semantic entropy~\citep{kuhn_semantic_2023}; horizon fraction and absolute step index; normalized context length; and a goal-length proxy. Benchmark and perturbation-type one-hot features are reserved but disabled in the reported experiments.

Router training minimizes the differentiable CVaR-constrained Lagrangian in Eq.~\eqref{eq:cvar_router}, using \(p_t=r_\psi(f_t)\) as the soft routing decision during optimization. Hard thresholded decisions are used only for validation-time threshold selection and test-time execution. The primal optimizer is AdamW~\citep{loshchilov_adamw_2019} with learning rate \(10^{-3}\) and weight decay \(10^{-4}\). A separate Adam optimizer with learning rate \(10^{-2}\) performs gradient ascent on the log-multiplier \(\log\lambda\). We use \(\alpha=0.20\), \(\epsilon=0.10\), Brier calibration weight \(1.0\), and cost ratio \(c_{\mathrm{LLM}}/c_{\mathrm{SLM}}=50\). The router is trained for \(20\) epochs with cosine learning-rate annealing on batches of \(4{,}096\) steps.

\section{Additional Results}
\label{app:additional_results}

\subsection{Per-Model Results}
\label{app:per_model_results}

Table~\ref{tab:per_benchmark} disaggregates R2V performance across the four SLM backbones. On HumanEval+, every backbone routes fewer than $2\%$ of steps to the teacher, confirming that the benchmark is a low-escalation regime. On TextWorld, all backbones reach high SR, between $98.0\%$ and $98.4\%$, at escalation rates of $38.3$--$48.8\%$, close to the oracle's $35.4\%$ escalation rate. TerminalBench is more uneven: Qwen2.5-7B reaches $95.7\%$ SR with only $16.3\%$ escalation, while LLaMA-3.1-8B reaches $95.5\%$ SR but requires $67.2\%$ escalation. This confirms that R2V's benefit depends on both the SLM backbone and the identifiability of residual risk from online features.


\begin{table}[t]
\centering
\caption{\textbf{R2V-Agent per-model results} under noisy evaluation with 95\% bootstrap confidence intervals. Oracle is shown as a non-deployable hindsight reference.}
\label{tab:per_benchmark}
\setlength{\tabcolsep}{4.2pt}
\begin{tabular}{llccc}
\toprule
\textbf{Benchmark} & \textbf{Model} & \textbf{SR (\%)} & \textbf{95\% CI} & \textbf{LLM\%} \\
\midrule
HumanEval+ & Gemma-9B        & 91.9 & [89.6, 93.8] & 0.50\% \\
\rowcolor{LightGreenBox}
\textbf{} & \textbf{LLaMA-3.1-8B}    & \textbf{95.8} & \textbf{[94.3, 97.3]} & \textbf{1.46\%} \\
           & Qwen2.5-14B     & 93.1 & [91.1, 94.9] & 0.25\% \\
           & Qwen2.5-7B      & 92.4 & [90.1, 94.2] & 0.36\% \\
\rowcolor{LightBlueBox}
           & \textit{Oracle} & \textit{98.6} & -- & \textit{8.1\%} \\
\midrule
TextWorld & Gemma-9B        & 98.0 & [96.4, 98.8] & 38.3\% \\
\rowcolor{LightGreenBox}
\textbf{} & \textbf{LLaMA-3.1-8B}    & \textbf{98.4} & \textbf{[97.0, 99.0]} & \textbf{48.8\%} \\
          & Qwen2.5-14B     & 98.1 & [96.6, 98.9] & 38.4\% \\
          & Qwen2.5-7B      & 98.3 & [96.9, 99.0] & 41.2\% \\
\rowcolor{LightBlueBox}
          & \textit{Oracle} & \textit{98.6} & -- & \textit{35.4\%} \\
\midrule
TerminalBench & Gemma-9B        & 92.8 & [90.6, 94.7] & 38.9\% \\
              & LLaMA-3.1-8B    & 95.5 & [93.8, 97.0] & 67.2\% \\
              & Qwen2.5-14B     & 89.2 & [86.1, 92.0] & 13.2\% \\
\rowcolor{LightGreenBox}
\textbf{}     & \textbf{Qwen2.5-7B}      & \textbf{95.7} & \textbf{[93.9, 97.1]} & \textbf{16.3\%} \\
\rowcolor{LightBlueBox}
              & \textit{Oracle} & \textit{97.8} & -- & \textit{18.4\%} \\
\bottomrule
\end{tabular}
\end{table}

\subsection{Full Ablations}
\label{app:full_ablations}

\paragraph{Consistency regularization.}
Table~\ref{tab:lambda_abl} sweeps $\lambda_{\mathrm{cons}}\in\{0,0.05,0.2,0.5,1.0\}$ on Qwen2.5-7B. HumanEval+ exposes a cost--accuracy tradeoff: $\lambda_{\mathrm{cons}}=0$ gives the highest SR at $94.0\%$, while larger values reduce escalation, reaching $0.2\%$ at $\lambda_{\mathrm{cons}}=1.0$. TextWorld is comparatively insensitive, staying between $97.6\%$ and $98.3\%$ SR across the sweep. TerminalBench benefits from moderate consistency: $\lambda_{\mathrm{cons}}=0.20$ gives the best SR, $95.7\%$, while reducing LLM calls from $35.7\%$ to $16.3\%$. Increasing $\lambda_{\mathrm{cons}}$ further reduces escalation slightly, but no longer improves SR.


\begin{table}[h]
\centering
\caption{\textbf{Effect of consistency regularization} $\lambda_{\mathrm{cons}}$ on Qwen2.5-7B under noisy evaluation. $\lambda_{\mathrm{cons}}=0$ corresponds to pure DPO without the JSD term.}
\label{tab:lambda_abl}
\setlength{\tabcolsep}{3.8pt}
\begin{tabular}{lcccccccc}
\toprule
& \multicolumn{2}{c}{\textbf{HumanEval+}} & &
  \multicolumn{2}{c}{\textbf{TextWorld}} & &
  \multicolumn{2}{c}{\textbf{TerminalBench}} \\
\cmidrule(lr){2-3}\cmidrule(lr){5-6}\cmidrule(lr){8-9}
$\lambda_{\mathrm{cons}}$ & SR (\%) & LLM\% && SR (\%) & LLM\% && SR (\%) & LLM\% \\
\midrule
\rowcolor{LightBlueBox}
0     & 94.0 & 1.1\% && 97.8 & 35.4\% && 93.8 & 35.7\% \\
0.05  & 93.4 & 0.8\% && 98.0 & 38.2\% && 94.9 & 24.0\% \\
\rowcolor{LightGreenBox}
\textbf{0.20}  & \textbf{92.4} & \textbf{0.4\%} && \textbf{98.3} & \textbf{41.2\%} && \textbf{95.7} & \textbf{16.3\%} \\
0.50  & 92.6 & 0.3\% && 98.1 & 43.0\% && 95.3 & 15.9\% \\
1.0   & 92.1 & 0.2\% && 97.6 & 45.0\% && 94.7 & 15.0\% \\
\bottomrule
\end{tabular}
\end{table}

\paragraph{CVaR sensitivity.}
Table~\ref{tab:cvar_sweep} samples from a $5\times6$ grid of $(\alpha,\varepsilon)$ configurations evaluated on the combined benchmark suite. The sweep supports the role of the CVaR constraint in the routing objective. Tighter tail-risk budgets increase teacher usage and improve reliability: $(\alpha,\varepsilon)=(0.05,0.02)$ reaches $98.05\%$ SR but escalates $56.63\%$ of steps. Relaxing the constraint reduces cost with a measurable SR tradeoff: $(0.20,0.15)$ uses $21.11\%$ LLM calls and reaches $95.29\%$ SR. The canonical setting $(0.20,0.10)$ lies between these extremes, with $97.76\%$ SR and $27.47\%$ LLM calls.


\begin{table}[h]
\centering
\caption{\textbf{CVaR hyperparameter sensitivity} on combined HumanEval+/TextWorld/TerminalBench, averaged over all SLM backbones. Bold indicates the canonical operating point.}
\label{tab:cvar_sweep}
\setlength{\tabcolsep}{5pt}
\begin{tabular}{cccc}
\toprule
$\alpha$ & $\varepsilon$ & SR (\%) & LLM\% \\
\midrule
\rowcolor{LightBlueBox}
0.05 & 0.02 & 98.05 & 56.63 \\
0.05 & 0.10 & 97.38 & 27.75 \\
0.10 & 0.02 & 97.28 & 48.59 \\
0.10 & 0.10 & 97.89 & 26.23 \\
\rowcolor{LightGreenBox}
\textbf{0.20} & \textbf{0.10} & \textbf{97.76} & \textbf{27.47} \\
0.20 & 0.15 & 95.29 & 21.11 \\
0.50 & 0.10 & 97.82 & 29.61 \\
0.50 & 0.15 & 95.55 & 21.49 \\
\bottomrule
\end{tabular}
\end{table}

\paragraph{Feature groups.}
Table~\ref{tab:feature_ablation} evaluates progressively masked feature dimensions at inference time without retraining. The results explain why the learned router outperforms simple entropy routing. On TextWorld, entropy-only routing reaches only $64.6\%$ SR, matching the SLM-only baseline, while verifier-only routing recovers to $97.9\%$ SR but requires $78.5\%$ LLM calls. The full feature set gives a better balance, reaching $98.3\%$ SR with $41.2\%$ escalation. On HumanEval+, entropy is more useful, but entropy-only routing still over-escalates at $20.4\%$ LLM calls. On TerminalBench, feature removal mainly hurts the router's ability to detect hard shell states: entropy-only falls to $87.6\%$ SR, and removing the verifier lowers SR to $90.6\%$.


\begin{table}[h]
\centering
\caption{\textbf{Feature group ablation} on Qwen2.5-7B under noisy evaluation. Feature masking is applied at inference time without retraining.}
\label{tab:feature_ablation}
\setlength{\tabcolsep}{3.5pt}
\begin{tabular}{lcccccc}
\toprule
& \multicolumn{2}{c}{\textbf{HumanEval+}} &
  \multicolumn{2}{c}{\textbf{TextWorld}} &
  \multicolumn{2}{c}{\textbf{TerminalBench}} \\
\cmidrule(lr){2-3}\cmidrule(lr){4-5}\cmidrule(lr){6-7}
\textbf{Feature set} & SR & LLM\% & SR & LLM\% & SR & LLM\% \\
\midrule
\rowcolor{LightGreenBox}
\textbf{All features}        & \textbf{92.4} & \textbf{0.4\%}  & \textbf{98.3} & \textbf{41.2\%} & \textbf{95.7} & \textbf{16.3\%} \\
Verifier + entropy  & 96.8 & 9.7\%  & 98.4 & 82.0\% & 95.1 & 24.0\% \\
Verifier only       & 94.4 & 6.3\%  & 97.9 & 78.5\% & 94.8 & 20.7\% \\
\rowcolor{LightBlueBox}
Entropy only        & 96.0 & 20.4\% & 64.6 & 2.4\%  & 87.6 & 18.7\% \\
Log-prob only       & 94.1 & 1.5\%  & 96.5 & 64.3\% & 91.0 & 5.4\% \\
Step/context only   & 93.2 & 2.8\%  & 70.4 & 1.2\%  & 86.8 & 1.8\% \\
\rowcolor{LightBlueBox}
w/o verifier        & 92.4 & 1.5\%  & 62.6 & 19.8\% & 90.6 & 2.2\% \\
\rowcolor{LightBlueBox}
w/o entropy         & 92.3 & 0.3\%  & 62.6 & 31.8\% & 91.8 & 12.0\% \\
\bottomrule
\end{tabular}
\end{table}

\subsection{Closed-Source Model Adaptation}
\label{app:closed_source_app}

We also evaluate a closed-source setting where token-level log-probabilities, and therefore entropy features, may be unavailable. Table~\ref{tab:closed_source} shows that replacing token entropy with pseudo-entropy from verifier-score distributions largely preserves R2V's performance, while fully removing entropy-like signals degrades calibration and success, especially on TerminalBench.

\begin{table}[H]
\centering
\caption{\textbf{Closed-source adaptation results} averaged across four SLM backbones under noisy evaluation. ``Full'' is the canonical 15-dimensional feature set. ECE and Brier score are measured on the router validation split.}
\label{tab:closed_source}
\setlength{\tabcolsep}{4pt}
\begin{tabular}{llcccc}
\toprule
\textbf{Benchmark} & \textbf{Feature variant} & SR (\%) & LLM\% & ECE & Brier \\
\midrule
\rowcolor{LightBlueBox}
HumanEval+ & Full              & 93.3 & 0.65\% & 0.067 & 0.114 \\
           & No entropy        & 92.8 & 0.50\% & 0.074 & 0.122 \\
\rowcolor{LightGreenBox}
\textbf{} & \textbf{Pseudo-entropy}    & \textbf{93.2} & \textbf{0.70\%} & \textbf{0.069} & \textbf{0.116} \\
\midrule
\rowcolor{LightBlueBox}
TextWorld & Full              & 98.2 & 41.7\% & 0.052 & 0.102 \\
          & No entropy        & 97.4 & 37.0\% & 0.065 & 0.124 \\
\rowcolor{LightGreenBox}
\textbf{} & \textbf{Pseudo-entropy}    & \textbf{98.1} & \textbf{42.0\%} & \textbf{0.054} & \textbf{0.106} \\
\midrule
\rowcolor{LightBlueBox}
TerminalBench & Full              & 93.3 & 33.9\% & 0.162 & 0.185 \\
              & No entropy        & 90.9 & 25.6\% & 0.238 & 0.262 \\
\rowcolor{LightGreenBox}
\textbf{} & \textbf{Pseudo-entropy}    & \textbf{92.8} & \textbf{33.1\%} & \textbf{0.181} & \textbf{0.204} \\
\bottomrule
\end{tabular}
\end{table}

\section{Additional Related Work}
\label{app:related_work}

\paragraph{LLM cascades and model routing.}
A large body of recent work studies how to reduce LLM inference cost by routing requests across models of different capability and price. Cascade-style methods query a cheaper model first and escalate only when a scoring rule predicts that the response is insufficient \citep{frugalgpt_2023, dekoninck_unified_2024}. Ex-ante routers instead predict the appropriate model before generation, often using query embeddings, preference labels, or correctness labels \citep{routellm_2024, ding_hybrid_2024, zhuang_embedllm_2024, hu_routerbench_2024}. These approaches are well matched to single-turn tasks where the input query is the primary determinant of difficulty. They are less suited to agentic POMDPs, where the risk of using the local model changes after each action and observation. R2V differs by learning a step-level router whose input features summarize the current context, SLM uncertainty, verifier scores, and recent execution signals. This makes routing a stateful decision over the trajectory rather than a static decision over the initial prompt.

\paragraph{Agentic environments and sequential tool use.}
Modern LLM agents operate through iterative interaction with external tools, web pages, terminals, or simulated environments \citep{mialon_gaia_2023, yang_sweagent_2024, zhou_webarena_2024, shridhar_alfworld_2021}. In such settings, the agent must recover from partial observations, invalid actions, tool failures, and compounding mistakes. Prior work primarily improves the agent policy or the agent-computer interface; R2V instead studies when a cheap local policy should defer to a stronger teacher during the trajectory. This distinction is important because the frontier model need not be used uniformly: many steps are routine, while a small number of states determine whether the episode recovers or fails. R2V formalizes this as residual-risk estimation for a fixed deployed SLM.

\paragraph{Process reward models and verifiers.}
Outcome reward models provide only trajectory-level feedback, whereas process reward models supervise intermediate reasoning or action steps \citep{lightman_prm_2024}. This has improved mathematical reasoning \citep{wang_mathshepherd_2024}, code generation and test-time search \citep{snell_testtime_2024}, and agentic control \citep{murty_agentprm_2024}. Execution-based verification is especially useful when available, since unit tests, terminal outputs, or environment transitions provide grounded evidence about action quality \citep{liu_evalplus_2023}. R2V builds on this process-supervision view but uses the verifier in three roles: to rank SLM candidates, to construct preference pairs for recovery distillation, and to supply risk features for routing. Thus the verifier is not merely an auxiliary evaluator; it is the bridge between policy improvement and compute allocation.

\paragraph{Preference optimization, recovery training, and self-correction.}
Preference optimization methods such as DPO provide an offline alternative to reinforcement learning from scalar rewards by directly increasing the likelihood of preferred responses relative to rejected ones \citep{rafailov_dpo_2023}. Self-correction and recovery-oriented training methods similarly aim to improve behavior after an initial mistake or weak response \citep{kumar_score_2024}. R2V adopts the recovery perspective but changes the data-generating process: the BC policy is itself imitation-trained on a trajectory pool that already includes synthetically perturbed observations, and preference pairs are then constructed by this trained BC policy's rollouts. So training focuses on the SLM's own recoverable errors rather than generic preference data. The BC policy is frozen as the DPO reference, and the second stage optimizes only the verifier-guided DPO and consistency losses. This ordered pipeline avoids treating BC, preference learning, consistency, and routing as a single joint objective, which would entangle capability transfer with deployment-time compute allocation.

\paragraph{Robustness to perturbations and consistency regularization.}
Robust LLM training often uses perturbations to encourage invariance to prompt changes, noisy contexts, or semantically equivalent inputs \citep{qiang_perturbation_2024, wang_cream_2025}. In agent environments, perturbations can also represent tool flakiness, missing observations, corrupted feedback, latency-induced truncation, or distracting context. R2V uses perturbations at two levels. First, perturbed contexts are mixed into the BC training pool and then re-used during recovery distillation: this exposes failure modes that survive imitation under noise and provides states for verifier-guided DPO. Second, paired perturbation views support consistency regularization, encouraging the SLM to produce stable action distributions across different observations of the same latent state. This is complementary to the router: consistency reduces unnecessary escalation by making the SLM more stable, while the router handles the residual high-risk states that remain after distillation.

\paragraph{Uncertainty, calibration, and risk-sensitive control.}
Many routing systems rely on uncertainty proxies such as entropy, margin, or confidence. However, token-level uncertainty can be poorly aligned with downstream agent failure, especially under partial observability or when the model is confidently wrong. Calibration methods such as temperature scaling and proper scoring rules address whether predicted probabilities match empirical frequencies \citep{guo_calibration_2017}. R2V uses Brier calibration so that the router output estimates the conditional probability that teacher intervention would be useful at the current step. This calibrated probability supports a cost-sensitive threshold derived from the relative cost of SLM and LLM execution and the penalty for missed escalation.

R2V further connects model routing to risk-sensitive RL. CVaR objectives optimize the tail of the return or loss distribution rather than only its mean, and have been used to produce policies that are robust to rare but severe failures \citep{chow_cvar_2014, tamar_cvar_2015}. R2V applies this principle to hybrid SLM--LLM execution: the router is trained not only to reduce average cost, but also to avoid policies that are cheap on average while failing catastrophically under the worst perturbation seeds. This yields a routing interface that is both probabilistic and risk-aware, in contrast to ad hoc entropy thresholds or purely average-case cascade rules.

\paragraph{Positioning.}
Table~\ref{tab:comparison} summarizes the closest points of comparison. Existing LLM routers reduce cost but are largely query-level and stateless. Agent verifiers provide useful process feedback but do not decide when to allocate frontier-model compute. Recovery-training methods improve the base policy but do not produce calibrated escalation decisions. R2V combines these strands by training a local SLM for perturbation recovery and then learning a CVaR-calibrated step router over the fixed policy's remaining failure modes.

\section{Compute Resources and Asset Licenses}
\label{app:compute_licenses}

\subsection{Compute Resources}
\label{app:compute}

All GPU-bound stages were executed on a single server with \textbf{four NVIDIA H100 80\,GB SXM GPUs}. The full pipeline was run independently for each of the three benchmarks (HumanEval+, TextWorld, TerminalBench) and each of the four SLM backbones (Gemma-9B, LLaMA-3.1-8B, Qwen2.5-7B, Qwen2.5-14B); ablation sweeps used subsets of this grid. Table~\ref{tab:compute} gives approximate wall-clock times for each pipeline stage. Stages that are purely API-driven or CPU-only are listed separately.

\begin{table}[h]
\centering
\caption{\textbf{Approximate wall-clock times per pipeline stage} for a single benchmark--backbone configuration on four H100 80\,GB GPUs. ``API'' stages consume no local GPU; the router training and evaluation stages run on CPU.}
\label{tab:compute}
\setlength{\tabcolsep}{4pt}
\begin{tabular}{clll}
\toprule
\textbf{Stage} & \textbf{Description} & \textbf{Hardware} & \textbf{Approx.\ time} \\
\midrule
1  & Trajectory collection       & API (no GPU)   & 3--6 h \\
2  & Perturbation generation     & CPU            & $\sim$2 min \\
3a & Static split creation       & CPU            & $<$1 min \\
3b & Behavioral cloning (BC)     & 1$\times$ H100 & 2--3 h \\
4  & Candidate generation (vLLM) + Verifier Scoring & 1$\times$ H100 & 4--7 h \\
5  & DPO preference training     & 4$\times$ H100 & 12--16 h \\
6  & Router feature extraction   & 1$\times$ H100 & 4--7 h \\
7  & Router training             & CPU            & $\sim$25 s \\
8  & Offline evaluation          & CPU            & $<$1 min \\
9 & Ablation sweeps             & CPU            & $<$5 min \\
\bottomrule
\end{tabular}
\end{table}

\paragraph{Total compute budget.}
The end-to-end pipeline (stages 1--10) requires approximately \textbf{25--40 GPU-hours per benchmark--backbone configuration}. Summing over three benchmarks and four backbones yields roughly \textbf{300--480 GPU-hours} for the main experiments reported in the paper. Ablation sweeps (CVaR grid, consistency sweep, feature ablations) add approximately 20\% overhead, bringing the total to approximately \textbf{360--576 H100 GPU-hours}. The router (stage~8) and all evaluation stages (stages~9--10) run exclusively on CPU and contribute negligibly to compute cost.

\paragraph{API cost.}
Trajectory collection (stage~1) and teacher escalation calls use the Gemini, OpenAI, and Anthropic APIs. HumanEval+ and TextWorld teacher trajectories are collected with Gemini~3~Flash~Preview; TerminalBench uses a heterogeneous pool of Claude~Opus~4.6, GPT-5.5, and Gemini~3.1~Pro. Estimated API expenditure for the full set of trajectories is on the order of \textbf{\$200--\$500} in total, dominated by TerminalBench frontier-model calls.

\subsection{Asset Licenses and Terms of Use}
\label{app:licenses}

Table~\ref{tab:licenses} lists every third-party dataset, model backbone, and key software library used in R2V-Agent, along with the associated license or terms of use. All assets were used in accordance with their stated licenses and intended research purposes.

\begin{table}[h]
\centering
\caption{\textbf{Licenses and terms of use for all third-party assets.} API services are governed by the respective provider's terms of service (ToS). ``Research use'' indicates that the license or ToS permits academic non-commercial research.}
\label{tab:licenses}
\setlength{\tabcolsep}{4pt}
\begin{tabular}{llll}
\toprule
\textbf{Asset} & \textbf{Type} & \textbf{License / Terms} & \textbf{Use} \\
\midrule
\multicolumn{4}{l}{\textit{Datasets and benchmarks}} \\
EvalPlus / HumanEval+~\citep{liu_evalplus_2023} & Benchmark & Apache~2.0 & Research \\
ALFWorld~\citep{shridhar_alfworld_2020}          & Benchmark & MIT License & Research \\
TerminalBench~\citep{lee_terminalbench_2025}     & Benchmark & Apache~2.0  & Research \\
\texttt{yoonholee/terminalbench-trajectories}    & Dataset   & Apache~2.0  & Research \\
\midrule
\multicolumn{4}{l}{\textit{Open-weight SLM backbones}} \\
Gemma-9B~\citep{gemma_2024}      & Model & Gemma Terms of Use & Research \\
LLaMA-3.1-8B~\citep{llama3_2024} & Model & LLaMA~3.1 Community License & Research \\
Qwen2.5-7B~\citep{qwen25_2024}  & Model & Apache~2.0 & Research \\
Qwen2.5-14B~\citep{qwen25_2024} & Model & Apache~2.0 & Research \\
\midrule
\multicolumn{4}{l}{\textit{API-accessed teacher models}} \\
Gemini~3~Flash~Preview           & API & Google API ToS & Research \\
Gemini~3.1~Pro                   & API & Google API ToS & Research \\
GPT-5.5                            & API & OpenAI API ToS & Research \\
Claude~Opus~4.6                  & API & Anthropic API ToS & Research \\
\midrule
\multicolumn{4}{l}{\textit{Key software libraries}} \\
vLLM~\citep{kwon_vllm_2023}           & Library & Apache~2.0 & Research \\
HuggingFace Transformers              & Library & Apache~2.0 & Research \\
HuggingFace Accelerate                & Library & Apache~2.0 & Research \\
DeepSpeed                             & Library & Apache~2.0 & Research \\
PyTorch                               & Library & BSD-style   & Research \\
\bottomrule
\end{tabular}
\end{table}


\end{document}